\documentclass[10pt]{article}

\usepackage{soul}

 \usepackage{palatcm, color, verbatim}
\usepackage{algorithmic, algorithm}
\usepackage{natbib}

\usepackage{hyperref}

\renewcommand{\cite}[1]{\citep{#1}}
\usepackage{mdwmath}
 \usepackage{amssymb}
 \usepackage{amsmath}
 \usepackage{amsthm}
\usepackage{epsfig}
\usepackage{subfigure}

\newtheorem{theo}{Theorem}

\newtheorem{defin}{Definition}
\newtheorem{lem}{Lemma}
\newtheorem{cor}{Corollary}
\newtheorem{rem}{Remark}

\newcommand{\rank}{\text{rank}}

\begin{document}

\title{Sparse PCA through Low-rank Approximations}

\author{
Dimitris S. Papailiopoulos \\
The University of Texas at Austin\\
\texttt{dimitris@utexas.edu}
  \and
Alexandros G. Dimakis\\
The University of Texas at Austin\\
\texttt{dimakis@austin.utexas.edu}
\and
Stavros Korokythakis\\
Stochastic Technologies\\
\texttt{stavros@stochastictechnologies.com}
}

\date{}
\maketitle

\begin{abstract}
We introduce a novel algorithm that computes the $k$-sparse principal component of a positive semidefinite matrix $A$.
Our algorithm is combinatorial and operates by examining a discrete set of special vectors lying in a low-dimensional eigen-subspace of $A$.
%Our algorithm  uses as a subroutine a modified version of the constant-rank solver for sparse PCA that was recently developed by \citep{asteris2011sparse}.
%We apply the constant rank solver on the first $d$ eigenvectors of the matrix, and quantify the approximation error in terms of its eigenvalues.
We obtain provable approximation guarantees that depend on the spectral decay profile of the matrix: the faster the eigenvalue decay, the better the quality of our approximation.
For example, if the eigenvalues of $A$ follow a power-law decay, we obtain a polynomial-time approximation 
algorithm for any desired accuracy. 

A key algorithmic component of our scheme is a combinatorial feature elimination step that is provably safe and in practice significantly reduces the running complexity of our algorithm.
We implement our algorithm and test it on multiple artificial and real data sets. Due to 
the feature elimination step, it is possible to perform sparse PCA on data sets consisting of millions of entries in a few minutes. 
Our experimental evaluation shows that our scheme is nearly optimal while finding very sparse vectors.
We compare to the prior state of the art and show that our scheme matches or outperforms previous algorithms in all tested data sets. 

\end{abstract}

\section{Introduction}
\label{sec:intro}

Principal component analysis (PCA) reduces the dimensionality of a data set by projecting it onto principal subspaces spanned by the leading eigenvectors of the sample covariance matrix.
The statistical significance of PCA partially lies in the fact that the principal components capture the largest possible data variance.
The first principal component (i.e., the first eigenvector) of an $n\times n$ matrix $A$ is the solution to
\begin{equation}
\underset{\|x\|_2 = 1}{\arg\max} \;x^TAx \nonumber
\end{equation}
where $A = SS^T$ and $S$ is the $n\times m$ data-set matrix consisting of $m$ data-points, or entries, each evaluated on $n$ features,  and $\|x\|_2$ is the $\ell_2$-norm of $x$.
PCA can be efficiently computed using the singular value decomposition (SVD).
The statistical properties and computational tractability of PCA renders it one of the most used tools in data analysis and clustering applications.

A drawback of PCA is that the generated vectors typically have very few zero entries, i.e., they are not {\it sparse}.
Sparsity is desirable when we aim for {\it interpretability} in the analysis of  principal components. 
An example where sparsity implies interpretability is document analysis, where principal components can be used to cluster documents and detect trends. 
When the principal components are sparse, they can be easily mapped to topics 
(e.g., newspaper article classification into politics, sports, etc.) using the few keywords in their support \citep{gawaltsparse,zhang2011large}. 
For that reason it is desirable to find sparse eigenvectors. 

\subsection{Sparse PCA}
Sparsity can be directly enforced in the principal components.
The sparse principal component $x_{*}$ is defined as
\begin{equation}
x_{*}=\underset{\|x\|_2 = 1, \|x\|_0=k}{\arg\max} x^TAx.
\label{xAx}
\end{equation}
The $\ell_0$ cardinality constraint limits the optimization over vectors with $k$ non-zero entries.
As expected, sparsity comes at a cost since the optimization in (\ref{xAx}) is NP-hard~\citep{moghaddam2006generalized} and hence computationally intractable in general. 

\subsection{Overview of main results}
We introduce a novel algorithm for sparse PCA that has a provable 
approximation guarantee. 
Our algorithm generates a $k$-sparse, unit length vector $x_d$ that gives an objective provably within a $1-\epsilon_d$ factor 
from the optimal: 
\[ 
x^T_d A x_d \geq 
(1-\epsilon_d) \, x^T_{*}Ax_{*}\]
with  
\begin{equation}
% 	\epsilon_d \leq \frac{\lambda_{d+1}}{\max\left\{\frac{k}{n}\cdot  \lambda_1,\;\lambda_1^{(1)}\right\}},
\epsilon_d \leq\min\left\{ \frac{n}{k}\cdot \frac{\lambda_{d+1}}{ \lambda_1},\;\;\frac{\lambda_{d+1}}{ \lambda^{(1)}_1} \right\},
	\label{egar}
\end{equation}
where $\lambda_i$ is the $i$th largest eigenvalue of $A$ and $\lambda_1^{(1)}$ is the maximum diagonal element of $A$.
For any desired value of the parameter $d$, 
our algorithm runs in time $O(n^{d+1}\log n+\text{SVD}(A,d))$, {\color{black}where $\text{SVD}(A,d)$ is the time to compute the $d$ principal eigenvectors of $A$}.
Our approximation guarantee is directly related to the spectrum of $A$: the greater the eigenvalue decay, the better the approximation. 
Equation (\ref{egar}) contains two bounds: one that uses the largest eigenvalue $\lambda_1$ and one that uses the largest diagonal element of $A$, $\lambda_1^{(1)}$. 
Either bound can be tighter, depending on the structure of the $A$ matrix. 

We subsequently rely on our approximation result to establish guarantees for considerably general families of matrices.

\subsubsection{Constant-factor approximation}
If we only assume that there is an arbitrary decay in the eigenvalues of $A$, \textit{i.e.}, there exists 
a constant $d=O(1)$ such that $\lambda_1>\lambda_{d+1}$, then we can obtain a constant-factor approximation 
guarantee for the linear sparsity regime. 
Specifically, we find a constant $\delta_0$ such that for all sparsity levels $k>\delta_0 \, n$ we obtain a constant approximation ratio for sparse PCA, partially solving the open problem discussed in~\citep{zhang2012sparse, d2012approximation}. This result easily follows from our main theorem. 
\subsubsection{PTAS under a power-law decay}
 When the data matrix spectrum exhibits a power-law decay, we can obtain a much stronger performance guarantee: we can solve sparse PCA for any desired accuracy $\epsilon$ in time polynomial in $n,k$ (but not in $\frac{1} {\epsilon}$). This is sometimes called a polynomial-time approximation scheme (PTAS).
Further, the power-law decay is not necessary: the spectrum does not have to follow exactly that decay, but only exhibit a substantial spectral drop after a few eigenvalues.

\subsubsection{Algorithmic details}
Our algorithm operates by scanning a low-dimensional subspace of $A$.
It first computes the leading eigenvectors of the covariance input matrix, and then scans this subspace for $k$ sparse vectors that have large explained variance.

The constant dimensional search is possible after a hyperspherical transformation of the $n$ dimensional problem space to one of  constant $d$ dimension. 
This framework was introduced by the foundational work of \cite{karystinos2010efficient} in the context of solving quadratic form maximization problems over $\pm1$ vectors.
This framework was consequently used  in \citep{asteris2011sparse}  to develop a constant rank solver that computes the sparse principal component of a constant rank matrix in polynomial time $O(n^{d+1})$.
We use as a subroutine a modified version of the solver of \cite{asteris2011sparse}, to examine a polynomial number of special vectors that lead to a sparse principal component which admits provable performance.
For matrices with nonnegative entries, 
we are able to tweak the solver and improve computation time by a factor of $2^d$.

Although the complexity of our algorithm is polynomial in $n$, the cost to run it on even moderately large sets with $n>1000$ becomes intractable even for small values of $d=2$, our accuracy parameter.
A key algorithmic innovation that we introduce is a provably safe feature elimination step that allows the scalability of our algorithm for data-sets with millions of entries.
We introduce a test that discards features that are provably not in the support of the sparse PC, in a similar manner as \cite{zhang2011large}, but using a different combinatorial criterion. 

\subsubsection{Experimental Evaluation} 
We evaluate and compare our algorithm against state of the art sparse PCA approaches on synthetic and real data sets. 
Our real data-set is a large Twitter collection of more than $10$ million tweets 
spanning approximately six months. 
We executed several experiments 
on various subsets of our data set: collections of tweets during a specific time-window, tweets that contained a specific word, etc.
Our implementation executes in less than one second for $50k-100k$ documents and in a few minutes for millions of documents, on a personal computer. 
Our scheme typically comes closer than 
 $90\%$ of the optimal performance,
even for $d<3$, and empirically outperforms previously proposed sparse PCA algorithms.

\subsection{Related Work}

 There has been a substantial volume of prior work on sparse PCA. 
Initial heuristic approaches used factor rotation techniques and thresholding of eigenvectors to obtain sparsity \citep{kaiser1958varimax,jolliffe1995rotation, cadima1995loading}.
Then, a modified PCA technique based on the LASSO (SCoTLASS) was introduced in \citep{jolliffe2003modified}.
In \citep{zou2006sparse}, a nonconvex regression-type approximation, penalized \`{a} la LASSO was used to produce sparse PCs.
A nonconvex technique
was presented in \citep{sriperumbudur2007sparse}.
In \citep{moghaddam2006spectral}, the authors used spectral arguments to motivate a greedy branch-and-bound approach, further explored in \citep{moghaddam2007fast}.
In \citep{shen2008sparse}, a similar technique to SVD was used employing sparsity penalties on each round of projections.
A significant body of work based on semidefinite programming (SDP) approaches was established in \citep{d2007direct,zhang2012sparse,d2008optimal}.
A variation of the power method was used  in \citep{journee2010generalized}.
When computing multiple PCs, the issue of deflation arises as discussed in~\citep{mackey2009deflation}.
In \citep{yuan2011truncated}, the authors introduced a very efficient sparse PCA approximation based on truncating the well-known power method to obtain the exact level of sparsity desired.
A fast algorithm based on Rayleigh quotient iteration was developed in \citep{kuleshov2013fast}.

{\color{black}Several guarantees are established under the statistical model of the spiked covariance.
In \citep{amini2008high}, the first theoretical optimality guarantees were established under the spiked covariance for diagonal thresholding and the SDP relaxation of \cite{d2007direct}.
In \citep{yuan2011truncated}, the authors  provide peformance guarantees for the truncated power method under specific assumptions of data model, similar to the restricted isometry property.
In \citep{d2012approximation} the authors provide detection guarantees under the single spike covariance model.
Then, in \citep{cai2012sparse} and \citep{cai2013optimal} the authors provide guarantees under the assumption of multiple spikes in the covariance.}

There has also been a significant effort in understanding the hardness of the problem.
Sparse PCA is NP-hard in the general case as it can be recast to the problem subset selection and the problem of finding the largest clique in a graph.
It is also suspected that it is computationally challenging  to recover the sparse spikes of a spiked covariance model, under optimal sample complexity as was shown in \citep{berthet2013computational}, \citep{berthet2013complexity}, and
\citep{berthet2012optimal}.
There, the problem of recovering the correct spike under the minimum possible sample complexity is connected to the problem of recovering a planted clique below the $\Theta(\sqrt{n})$ barrier.

Despite this extensive literature, to the best of our knowledge, there are very few provable approximation guarantees for the optimization version of the sparse PCA problem, and usually under restricted statistical data models~\citep{amini2008high, yuan2011truncated, d2012approximation, cai2013optimal}.

\section{Sparse PCA through Low-rank Approximations}
\subsection{Proposed Algorithm}
\label{sec:spca}

Our algorithm is technically involved and for that reason we start with a high-level informal description. 
For any given accuracy parameter $d$ we follow the following steps:

\textbf{Step 1:} {\it Obtain $A_d$, a rank-$d$ approximation of $A$.}\\
We obtain $A_d$, the best-fit rank-$d$ approximation of $A$,
by keeping the first $d$ terms of its eigen-decomposition:
\begin{equation}
A_d = \sum_{i=1}^d \lambda_i v_i v_i^T,\nonumber
\end{equation}
where $ \lambda_i$ is the $i$-th largest eigenvalue of $A$ and $ v_i$ the corresponding eigenvector.

\textbf{Step 2:} {\it Use $A_d$ to obtain $O(n^d)$ candidate supports.}\\
For any matrix $A$, we can exhaustively search for the optimal $x_*$ by checking all $n\choose k$ possible $k\times k$ submatrices of $A$:
$x_*$ is the $k$-sparse vector with the same support as the submatrix of $A$ with the maximum largest  eigenvalue.
However, we show how sparse PCA can be efficiently solved on $A_d$ if the rank $d$ is constant with respect to $n$, using the machinery of \cite{asteris2011sparse}.
The key technical fact proven there is that there are \textit{only} $O(n^d)$ candidate supports that need
to be examined. That is, a {\it set  of candidate supports}
  $\mathcal{S}_{d} = \left\{\mathcal{I}_1,\ldots,\mathcal{I}_T\right\}$,
where $\mathcal{I}_t$ is a subset of $k$ indices from $\{1,\ldots, n\}$, contains the optimal support.
The number of these supports is\footnote{In fact, in our proof we show a better dependency on $d$, which however 
has a more complicated expression.
\begin{equation}
	|\mathcal{S}_{d}| \leq 2^{2d} {n \choose d}.\nonumber
\end{equation}
The above set $\mathcal{S}_d$ is efficiently created by the \textit{Spannogram algorithm} described in the next subsection. 
}

\textbf{Step 3:} {\it Check each candidate support from $\mathcal{S}_d$  on $A$.}\\
For a given support $\mathcal{I}$ it is easy to find the best vector supported on $\mathcal{I}$: 
it is the leading eigenvector of the principal submatrix of $A$, with rows and columns indexed by $\mathcal{I}$. 
In this step, we check all the supports in $\mathcal{S}_{d}$ \textit{on the original matrix} $A$ and output the best.
Specifically, define $A_\mathcal{I}$ to be the zeroed-out version of $A$, except on the support $\mathcal{I}$. 
That is, $A_\mathcal{I}$ is an $n\times n$ matrix with zeros everywhere except for the principal submatrix indexed by $\mathcal{I}$. If $i\in \mathcal{I}$ and $j\in \mathcal{I}$, then $A_\mathcal{I}= A_{ij}$, else it is $0$. 
Then, for any $A_\mathcal{I}$ matrix, with $\mathcal{I}\in\mathcal{S}_d$, we compute its largest eigenvalue and corresponding eigenvector.

\textbf{Output:}\\
Finally, we output  the $k$-sparse vector $x_d$ that is the principal eigenvector of the $A_\mathcal{I}$ matrix, $\mathcal{I}\in\mathcal{S}_d$, with the largest maximum eigenvalue.
We refer to this approximate sparse PC solution as the {\it rank-$d$ optimal solution}.

The exact steps of our algorithm are given in the pseudocode tables denoted as Algorithm 1 and 2. 
The spannogram subroutine, i.e., Algorithm 2, computes the $T$ candidate supports in $\mathcal{S}_d$, and is presented and explained in Section \ref{sec:spannogram}. 
The complexity of our algorithm is equal to calculating $d$ leading eigenvectors of $A$ ($\mathcal{O}(SVD(A,d))$), running the spannogram algorithm ($\mathcal{O}(n^{d+1}\log n)$), and finding the leading eigenvector of $O(n^d)$ matrices of size $k\times k$ ($\mathcal{O}(n^d k^3)$).
 Hence, the total complexity is $O(n^{d+1}\log n+n^dk^3+SVD(A,d))$.

\begin{algorithm}[h]
   \caption{Sparse PCA via a rank-$d$ approximation}
   \label{alg:example}
\begin{algorithmic}[1]
   \STATE {\bfseries Input:} $k$, $d$, $A$
\STATE $p \leftarrow \text{1 if $A$ has nonnegative entries, 0 if mixed}$
   \STATE $A_d \leftarrow \sum_{i=1}^d \lambda_i v_i v_i^T$
\STATE $\hat A_d \leftarrow \text{feature\_elimination}(k,p,A_d)$
\STATE ${\mathcal S}_{d}\leftarrow$ Spannogram$\left(k, p, {\hat A}_d\right)$
   \FOR{ each $\mathcal{I}\in\mathcal{S}_d$} 
   \STATE Calculate $\lambda_1(A_\mathcal{I})$
   \ENDFOR
\STATE $\mathcal{I}^{\text{opt}}_d=\arg\max_{I\in\mathcal{S}_d}\lambda_1(A_{\mathcal{I}})$
\STATE $\text{OPT}_d=\lambda_1(A_{I^{\text{opt}}_d})$
 \STATE$x_d^{\text{opt}}\leftarrow$ the principal eigenvector of $A_{\mathcal{I}^{\text{opt}}_d}$.
\STATE {\bf Output:} $x_d^{\text{opt}}$
\end{algorithmic}
 \end{algorithm}

\textbf{Elimination Step:} 
This step is run before Step 2.
By using a feature elimination subroutine we can identify that certain variables provably cannot be in the support of $x_d$, the rank-$d$ optimal sparse PC.
We have a test which is related to the norms of the rows of $V_d$ that identifies which of the $n$ rows cannot be in the optimal support.
We use this step to further reduce the number of candidate supports  $|\mathcal{S}_{d}|$.
The elimination algorithm is very important when it comes to large scale data sets.
Without the elimination step, even the rank-2 version of the algorithm becomes intractable for $n>10^4$.
However, after running the subroutine we empirically observe that even for $n$ that is in the orders of $10^6$ the elimination strips down the number of features to only around $50-100$ for values of $k$ around $10$.
This subroutine is presented in detail in the Appendix.

\subsection{Approximation Guarantees}
The desired sparse PC is
\begin{equation}
x_{*}=\underset{\|x\|_2 = 1, \|x\|_0=k}{\arg\max} x^TAx.\nonumber
\end{equation}
We instead obtain the $k$-sparse, unit length vector $x_d$ which gives an objective 
\[
x^T_d A x_d =  \underset{\mathcal{I} \in \mathcal{S}_d}{\max} \, \lambda(A_\mathcal{I}).
\]
We measure the quality of our approximation using the standard approximation factor:
\begin{equation} 
\rho_d=\frac{ x^T_d A x_d }{x^T_* A x_* }=
 \frac{ \underset{\mathcal{I} \in \mathcal{S}_d}{\max}\lambda(A_\mathcal{I})}{\lambda_1^{(k)}},\nonumber
\end{equation}
where $\lambda_1^{(k)} = x^T_* A x_*$ is the $k$-sparse largest eigenvalue of $A$.\footnote{Notice that the $k$-sparse largest eigenvalue of $A$ for $k=1$ denoted by $\lambda_1^{(1)}$ is simply the largest element on the diagonal of $A$.}
Clearly, $\rho_d\le 1$ and as it approaches $1$, the approximation becomes tighter.
Our main result follows:
\begin{theo}
For any $d$, our algorithm outputs $x_d$, where $||x_d||_0$=k, $||x_d||_2$=1 and    
\[ 
x^T_d A x_d \geq (1-\epsilon) \, x^T_{*}Ax_*,
\]
with an error bound
\begin{equation}
	\epsilon_d \leq \frac{\lambda_{d+1}}{\lambda^{(k)}_1}\le \min\left\{\frac{n}{k}\frac{\lambda_{d+1}}{ \lambda_1}, \frac{\lambda_{d+1}}{\lambda_1^{(1)}} \right\}.\nonumber
\end{equation}
\end{theo}
\begin{proof}
The proof can be found in the Appendix. 
The main idea is that we  obtain {\it i}) an upper bound on the performance loss using $A_d$ instead of $A$ and {\it ii)} a lower bound for $\lambda_1^{(k)}$.
\end{proof}

We now use our main theorem to provide the following model specific approximation results.
\begin{cor}
Assume that for some constant value $d$, there is an eigenvalue decay $\lambda_1>\lambda_{d+1}$ in $A$.
Then there exists a constant $\delta_0$ such that for all sparsity levels $k>\delta_0 n$ we obtain a constant approximation ratio.
\end{cor}

\begin{cor}
Assume that the first $d+1$ eigenvalues of $A$ follow a power-law decay, i.e.,
$\lambda_i = C i^{-\alpha}$,
for some $C,\alpha>0$.
Then, for any $k = \delta n$ and any $\epsilon>0$ we can get a $(1-\epsilon)$-approximate solution $x_d$ in time 
$O\left(n^{1/(\epsilon\delta)^\alpha+1}\log n\right)$.
\end{cor}

The above corollaries can be established by plugging in the values for $\lambda_i$ in the error bound.
We find the above families of matrices interesting, because in practical data sets (like the ones we tested), we observe a significant  decay in the first eigenvalues of $A$ which in many cases follows a power law.
The main point of the above approximability result is that any matrix with decent decay in the spectrum endows a good sparse PCA approximation.

\section{The Spannogram Algorithm}
\label{sec:spannogram}

In this section, we describe how the Spannogram algorithm constructs the candidate supports in $\mathcal{S}_d$ and explain why this set has tractable size. 
We build up to the general algorithm by explaining special cases that are easier to understand. 

\subsection{Rank-$1$ case}
\label{subsec:rank1}

Let us start with the rank $1$ case, i.e., when $d=1$. 
For this case 
\[ 
A_1= \lambda_1 v_1 v_1^T.
\]
Assume, for now, that all the eigenvector entries are unique. This simplifies tie-breaking issues that are formally addressed by a perturbation lemma in our Appendix.
For the rank-$1$ matrix $A_1$, a simple thresholding procedure solves sparse PCA: 
simply keep the $k$ largest entries of the eigenvector $v_1$. 
% Further, notice that there is only one candidate support set, the set of the k largest entries of $v_1$. 
Hence, in this simple case $\mathcal{S}_1$ consists of only $1$ set. 

To show this, we can rewrite \eqref{xAx} as
\begin{equation}
 \max_{x\in \mathbb{S}_{k}}x^TA_1x= 
\lambda_1\cdot \max_{x\in \mathbb{S}_{k}}\left( v_1^Tx\right)^2 = 
\lambda_1\cdot \max_{x\in \mathbb{S}_{k}}\left(\sum_{i=1}^n v_{1i}x_i\right)^2 ,
\label{rank1}
\end{equation}
where $\mathbb{S}_k$ is the set of all vectors $x \in \mathbb{R}^n$ with $||x||_2=1$ and $||x||_0=k$.
We are trying to find a $k$-sparse vector $x$ that maximizes the inner product with a given vector $v_1$. 
It is not hard to see that this problem is solved by sorting the absolute elements of the eigenvector $v_1$ and keeping the support of the $k$ entries in $v_1$ with the largest amplitude.

% For our consequent derivations we will use the following definition.
\begin{defin}
Let $\mathcal{I}_k(v)$ denote the set of indices of the top $k$ largest absolute values of a vector $v$. 
\end{defin}

We can conclude that for the rank-1 case, the optimal $k$-sparse PC for $A_1$ will simply be the $k$-sparse vector that is co-linear to the $k$-sparse vector induced on this unique candidate support. 
This will be the only rank-1 candidate optimal support 
\[ 
\mathcal{S}_1= \{ \, \mathcal{I}_k (v_1) \,  \}.
\]

\subsection{Rank-$2$ case}

Now we describe how to compute $\mathcal{S}_2$ using the constant rank solver of \cite{asteris2011sparse}.
This is the first nontrivial $d$ which exhibits the details of the spannogram algorithm.
We have the rank 2 matrix 
\[ 
A_2 =\sum_{i=1}^2 \lambda_i v_i v^T_i= V_2 V^T_2,
\]
where $ V_2=\left[\sqrt{\lambda_1}\cdot v_1\;\;\sqrt{\lambda_2}\cdot v_2\right]$.
We can rewrite (\ref{xAx}) on $A_2$ as
 \begin{equation}
  \max_{x\in \mathbb{S}_k} x^TA_2x=\max_{x\in \mathbb{S}_k} {\left\| V^T_2x\right\|_2^2}.
 \label{rank2}
\end{equation}
In the rank-1 case 
we could write the quadratic form maximization as a simple maximization of a dot product 
\[
\max_{x\in \mathbb{S}_k} x^T A_1 x =  \max_{x\in \mathbb{S}_{k}}\left( v_1^Tx\right)^2.
\]
Similarly, we will prove that in the rank-2 case we can write
\[
\max_{x\in \mathbb{S}_k} x^T A_2 x =  \max_{x\in \mathbb{S}_{k}}\left( v_c^Tx\right)^2,
\]
for some {\it specific} vector $v_c$ in the span of the eigenvectors $v_1,v_2$; this will be very helpful in solving the problem efficiently.

To see this, let $ c$ be a $2\times 1$ unit length vector, i.e., $\| c\|_2=1$.
Using the Cauchy-Schwartz inequality for the inner product of $ c$ and $ V^T_2x$ we obtain
 $\left( c^T V^T_2x\right)^2 \le 
% \|{ c}\|^2_2\cdot \| V^T_2x\|^2_2
%  =
 \| V^T_2x\|^2_2,\nonumber$
where equality holds, if and only if, $ c$ is co-linear to $ V_2^Tx$.
By the previous fact, we have a variational characterization of the $\ell_2$-norm:
\begin{equation}
 \| V^T_2x\|_2^2=\max_{\| c\|_2=1}\left( c^T V^T_2x\right)^2.
\label{normvar}
\end{equation}
We can use \eqref{normvar} to rewrite \eqref{rank2} as
\begin{align}
\max_{x\in\mathbb{S}_k,\| c\|_2=1}\left( c^T V^T_2x\right)^2&=\max_{x\in\mathbb{S}_k}\max_{\|c\|_2=1}\left( v_{ c}^Tx\right)^2\nonumber\\
&=\max_{\|c\|_2=1}\max_{x\in\mathbb{S}_k}\left( v_{ c}^Tx\right)^2,
\label{cVx}
\end{align}
where $ v_ c= V_2 c$.

We would like to note two important facts here.
The first is that for all unit vectors $c$, $ v_ c= V_2 c$ generates all vectors in the span of $ V_2$ (up to scaling factors).
The second fact is that if we fix  $c$, then the maximization $\max_{x\in \mathbb{S}_k}\left( v_ c^Tx\right)^2$ is a rank-$1$ instance, similar to \eqref{rank1}.
Therefore, for each fixed unit vector $c$ there will be one candidate support 
(denote it by $\mathcal{I}_k( V_2 c)$) to be added in $\mathcal{S}_2$.

If we could collect all possible candidate supports $\mathcal{I}_k( V_2 c)$ in
\begin{equation}
\mathcal{S}_2 = \bigcup_{ c\in \mathbb{R}^{2\times1}, \| c\|_2=1} \left\{\mathcal{I}_k( V_2 c)\right\},
\end{equation}
then we could solve exactly the sparse PCA problem on $A_2$: we would simply need to test all locally optimal solutions obtained from each support in $\mathcal{S}_2$ and  keep the one with the maximum metric.
The issue is that there are infinitely many $v_c$ vectors to check. 
Naively, one could think that all possible 
 $n \choose k$ 
$k$-supports could appear 
for some $v_c$ vector.
The key combinatorial fact is that if a vector $v_c$ lives in a two dimensional subspace, there are tremendously fewer possible supports\footnote{This is a special case of the general $d$ dimensional lemma of \cite{asteris2011sparse} (found in the Appendix), but we prove the special case to simplify the presentation.}:
\[
| \mathcal{S}_2 | \leq 4 {n \choose 2}.
\]

\subsubsection{ Spherical variables and the spannogram}
Here we use a transformation of our problem space into a $2$-dimensional space as was done in \cite{karystinos2010efficient}.
The transformation is performed through spherical variables that enable us to visualize the 2-dimensional span of $V_2$.
For the rank-$2$ case, we have a single phase variable $\phi \in \varPhi = \left(-\frac{\pi}{2},\; \frac{\pi}{2}\right]$ and use it to rewrite $ c$, without loss of generality, as
\begin{equation}
  c = \begin{bmatrix}
\sin\phi\\
\cos\phi
\end{bmatrix},\nonumber
\end{equation}
which is again unit norm and for all $\phi$ it scans all\footnote{Note that we restrict ourselves to $\left(-\frac{\pi}{2},\; \frac{\pi}{2}\right]$, instead of the whole $\left(-\pi,\; \pi\right]$ angle region. 
First observe that the vectors in the complement of $\varPhi$ are opposite to the ones evaluated on $\varPhi$.
Omitting the opposite vectors poses no issue due to the squaring in \eqref{rank2}, i.e., vectors $ c$ and $- c$ map to the same solutions.} $2\times1$ unit vectors.
 Under this characterization, we can express $ v_ c$ in terms of $\phi$ as
 \begin{equation}
 v(\phi)= V_2 c = \sin\phi\cdot \sqrt{\lambda_1} v_1+\cos\phi\cdot \sqrt{\lambda_2} v_2.
\end{equation}
Observe that each element of $ v(\phi)$ is a continuous curve in $\phi$:
\begin{equation}
[ v(\phi)]_i = \left[ \sqrt{\lambda_1} v_1\right]_{i}\sin(\phi)+\left[\sqrt{\lambda_2} v_2\right]_{2}\cos(\phi),\nonumber
 \end{equation}
for all $i=1,\ldots, n$.
Therefore, the support set of the $k$ largest absolute elements of $ v(\phi)$ (i.e., $\mathcal{I}_k( v(\phi))$) is itself a function of $\phi$. 
 \begin{figure}[t]
\centerline{
\includegraphics[width=0.5\columnwidth]{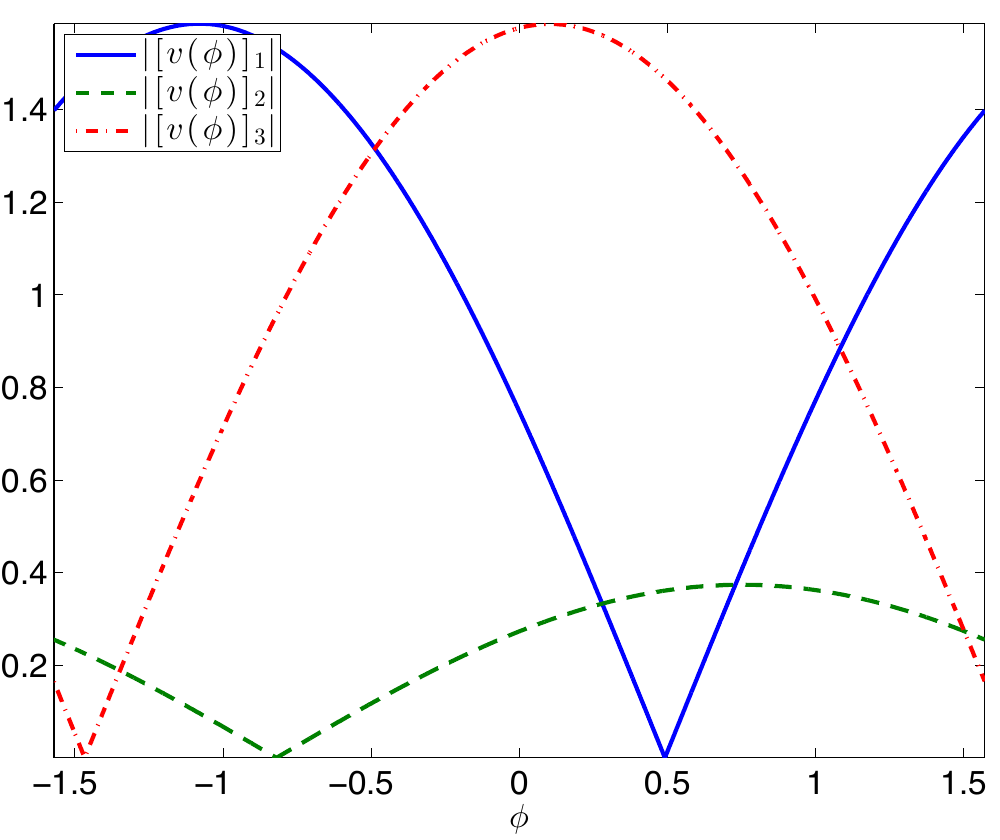}
}
\caption{A rank-$2$ spannogram for a $V_2$ matrix with  $n=3$.}
\label{fig:2dspan}
\end{figure}

In Fig.~\ref{fig:2dspan}, we draw an example plot of $3$ (absolute) curves $|[v(\phi)]_i|$, $i = 1,2,3$, from a randomly generated matrix $ V_2$.
We call this a \textit{spannogram}, because at each $\phi$, the values of curves correspond to the absolute values of the elements in the column span of $ V_2$.
Computing $[v(\phi)]_i$ for all $i,\phi$ is equivalent to computing the span of $ V_2$.
From the spannogram in Fig.~\ref{fig:2dspan}, we can see that the continuity of the curves implies a local invariance property of the support sets $\mathcal{I}( v(\phi))$, around a given $\phi$. 
%That is, we expect that $\mathcal{I}_k( v(\phi\pm\epsilon))=\mathcal{I}_k( v(\phi))$, for a sufficiently small $\epsilon>0$.
As a matter of fact, a support set  $\mathcal{I}_k( v(\phi))$  changes, {\it if and only if},  the respective sorting of two absolute elements $|[ v(\phi)]_i|$ and $|[ v(\phi)]_j|$ changes.
Finding these interesection points $|[ v(\phi)]_i|=|[ v(\phi)]_j|$ is the key to find all possible support sets.

There are $n$ curves and each pair intersects on exactly two points.\footnote{As we mentioned, we assume that the 
curves are in ``general position,'' i.e., no three curves intersect at the same point and this can be enforced by a small perturbation argument presented in the Appendix.} 
Therefore, there are exactly $2{{n}\choose{2}}$ intersection points. 
The intersection of two absolute curves are {\it exactly} two points $\phi$ that are a solution to $[ v(\phi)]_{i}=[ v(\phi)]_{j}$ and $[ v(\phi)]_{i}=-[ v(\phi)]_{j}$.
These are the {\it only} points where local support sets might change.
These $2{{n}\choose{2}}$  intersection points partition $\Phi$ in $2{{n}\choose{2}}+1$ regions within which the 
top $k$ support sets remain invariant.

\subsubsection{Building $\mathcal{S}_2$}
To build  $\mathcal{S}_2$, we need to {\it i)} determine all $ c$ intersection vectors that are defined at intersection points on the $\phi$-axis and {\it ii)} compute all distinct locally optimal support sets $\mathcal{I}_k( v_ c)$.
To determine an intersection vector we need to solve all $2{n\choose 2}$  equations $[ v(\phi)]_{i}=\pm[ v(\phi)]_{j}$ for all pairs $i,j\in[n]$.
This yields
$[ v(\phi)]_{i}=\pm[ v(\phi)]_{j} \Rightarrow  e^T_i V c = \pm e^T_j V c$, that is
% $\left(\left[ V_2\right]_{i,1}\mp\left[ V_2\right]_{j,1}\right)\sin(\phi)=\left(\pm\left[ V_2\right]_{j,2}-\left[ V_2\right]_{i,2}\right)\cos(\phi)$, that is
\begin{align}
\left( e^T_i\pm e^T_j\right)\hspace{-0.05cm}V\hspace{-0.05cm} c \hspace{-0.05cm}=\hspace{-0.05cm} 0
\Rightarrow  \hspace{-0.05cm}c\hspace{-0.05cm} =\hspace{-0.05cm} \text{nullspace}\left(\left( e^T_i\pm e^T_j\right) V\right).
\label{rank2Sol}
% \phi^{(l)}_{i,j} = \tan^{-1}\left(\frac{\left[ V_2\right]_{j,2}-(-1)^l\left[ V_2\right]_{i,2}}{(-1)^l\left[ V_2\right]_{i,1}-\left[ V_2\right]_{j,1}}\right), l=1,2.
% \left[\left(\left[ V_2\right]_{i,1}\mp\left[ V_2\right]_{j,1}\right)\;\;\;\left(\pm\left[ V_2\right]_{j,2}-\left[ V_2\right]_{i,2}\right)\right] c = 0\nonumber
\end{align}
Since $ c$ needs to be unit norm, we simply need to normalize the solution $ c$.
We will refer to the intersection vector calculated on the $\phi$ of the intersection of two curves $i$ and $j$ as $ c^+_{i,j}$ and $ c^-_{i,j}$, depending on the corresponding sign in \eqref{rank2Sol}.
For the intersection vectors $ c^+_{i,j}$ and $ c^-_{i,j}$ we compute $\mathcal{I}_k( V_2 c^+_{i,j})$ and $\mathcal{I}_k( V_2 c^-_{i,j})$.
Observe that since the $i$ and $j$ curves are equal on the intersection points, there is no prevailing sorting among the two corresponding elements $i$ and $j$ of $ V_2 c^+_{i,j}$ or $ V_2 c^-_{i,j}$.
Hence, for each intersection vector $ c_{i,j}^+$ and $ c_{i,j}^-$, we create two candidate support sets, one where element $i$ is larger than $j$, and vice versa.
This is done to secure that both support sets, left and right of the $\phi$ of the intersection, are included in $\mathcal{S}_2$.
With the above methodology, we can compute all possible $\mathcal{I}_k( V_2 c)$ rank-2 optimal candidate sets and we obtain
\begin{equation}
|\mathcal{S}_2| \le 4 {n\choose 2}= O(n^2).\nonumber
\end{equation}

The time complexity to build $\mathcal{S}_2$ is then equal to sorting $n\choose 2$ vectors and solving $2{n\choose 2}$ equations in the $2$ unknowns of $ c^+_{i,j}$ and $ c^+_{i,j}$.
That is, the total complexity is equal to ${n\choose 2}n\log n+{n\choose 2}2^2=O\left(n^3\log n\right)$.

\begin{rem}
The spannogram algorithm operates by simply solving systems of equations and sorting vectors. 
It is not iterative nor does it attempt to solve a convex optimization problem. 
Further, it computes solutions that are {\it exactly} $k$-sparse, where the desired sparsity can be set a-priori.  
\end{rem}

The spannogram algorithm presented here is a subroutine that can be used to find the leading sparse PC of $A_d$ in polynomial time.
The general rank-$d$ case is given as Algorithm 2.
The details of the constant rank algorithm, the elimination step, and tune-ups for matrices with nonnegative entries can be found in the Appendix.

\section{Experimental Evaluation}
\label{sec:exp}

We now empirically evaluate the performance of our algorithm and compare it to the full regularization path greedy approach (FullPath) of \citep{daspremont2007full}, the generalized power method (GPower) of \citep{journee2010generalized}, and the truncated power method (TPower) of \citep{yuan2011truncated}.
We omit the DSPCA semidefinite approach of \citep{d2007direct}, since the FullPath algorithm is experimentally shown to have similar or better performance \citep{d2008optimal}.

We start with a synthetic experiment: we seek to estimate the support of the first two sparse eigenvectors of a covariance matrix from sample vectors.
We continue with testing our algorithm on gene expression data sets.
Finally, we run experiments on a large-scale document-term data set, comprising of millions of Twitter posts.

\subsection{Spiked Covariance Recovery}
We first test our approximation algorithm on an artificial data set generated in the same manner as in \citep{shen2008sparse, yuan2011truncated}.
We consider a covariance matrix $\Sigma$, which has two sparse eigenvectors with very large eigenvalues and the rest of the eigenvectors correspond to small eigenvalues.
Here, we consider $\Sigma = \sum_{i=1}^n \lambda_i v_iv_i^T$ with $\lambda_1 = 400, \lambda_2 = 300, \lambda_3 = 1, \ldots, \lambda_{500} = 1$.
where the first two eigenvectors are sparse and each has $10$ nonzero entries and non-overlapping supports.
The remaining eigenvectors are picked as $n-2$ orthogonal vectors in the nullspace of $[v_1\;v_2]$.

We have two sets of experiments, one for few samples and one for extremely few.
First, we generate $m=50$ samples of length $n=500$ distributed as zero mean Gaussian  with covariance matrix $\Sigma$ and repeat the experiment $5000$ times.
We repeat the same experiment for $m=5$.
We compare our rank-1 and rank-2 algorithms against FullPath, GPower with $\ell_1$ penalization and $\ell_0$ penalization, and TPower.
After estimating the first eigenvector with $\tilde{v}_1$, we deflate $A$ to obtain $A'$.
We use the projection deflation method \citep{mackey2009deflation} to obtain
$A' = (I-\tilde{v}_1\tilde{v}_1^T)A(I-\tilde{v}_1\tilde{v}_1^T)$ and work on it
to obtain  $\tilde{v}_2$, the second estimated eigenvector of $\Sigma$.

In Table 1, we report the probability of correctly recovering the supports of $v_1$ and $v_2$: if both estimates $\tilde{v}_1$ and $\tilde{v}_2$ have matching supports with the true eigenvectors, then the recovery is considered successful.
\begin{table}[h]
 \centerline{
{\small
\begin{tabular}{c|c|c|c}
\hline
\hline
&  &$500\times 50$  & $500\times 5$\\
\hline
& $k$  & $p_{\text{rec.}}$ & $p_{\text{rec.}}$ \\
\hline
\hline
PCA+thresh.& $10$ & $.98$& $0.85$ \\
\hline
GPower-$\ell_0$  ($\gamma=0.8$)& $10$ &  $1$& $0.33$ \\
\hline
GPower-$\ell_1$ ($\gamma=0.8$)& $10$ &  $1$& $0.33$ \\
\hline
FullPath& $10$  &  $1$& $0.96$ \\
\hline
TPower& $10$ &  $1$ &  $0.96$\\
\hline
Rank-2 approx.& $10$   & $1$&  $0.96$ \\
\hline
\hline
\end{tabular}
}
}
\caption{Performance results on the spiked covariance model, where $p_{\text{rec.}}$ represents the recovery probability of the correct supports of the two sparse eigenvectors of $\Sigma$.}
\end{table}

In our experiments for $m=50$, all algorithms were comparable and performed near-optimally, apart from the rank-1 approximation (PCA+thresholding).
The success of our rank-$2$ algorithm can be in parts suggested by the fact that the true covariance $\Sigma$ is almost rank 2: it has very large decay between its 2nd and 3rd eigenvalue.
The average approximation guarantee that we obtained from the generating experiments for the rank 2 case and for $m=50$ was
$x_2^TAx_2\ge 0.7\cdot x_*Ax_*^T$, that is before running our algorithm, we know that it could on average perform at least 70\% as good as the optimal solution.
For $m=5$ samples we observe that the performance of the rank-1 and GPower methods decay and FullPath, TPower, and rank-2 find the correct support with probability approximately equal to $96\%$.
This overall decay in performance of all schemes is due to the fact that $5$ samples are not sufficient for a perfect estimate.
Interesting tradeoffs of sample complexity and probability of recovery where derived in \citep{amini2008high}.
Conducting a theoretical analysis for our scheme under the spiked covariance model is left as an interesting future direction.

\subsection{Gene Expression Data Set}
\begin{figure}[h]
\begin{center}
   \includegraphics[width=0.5\columnwidth]{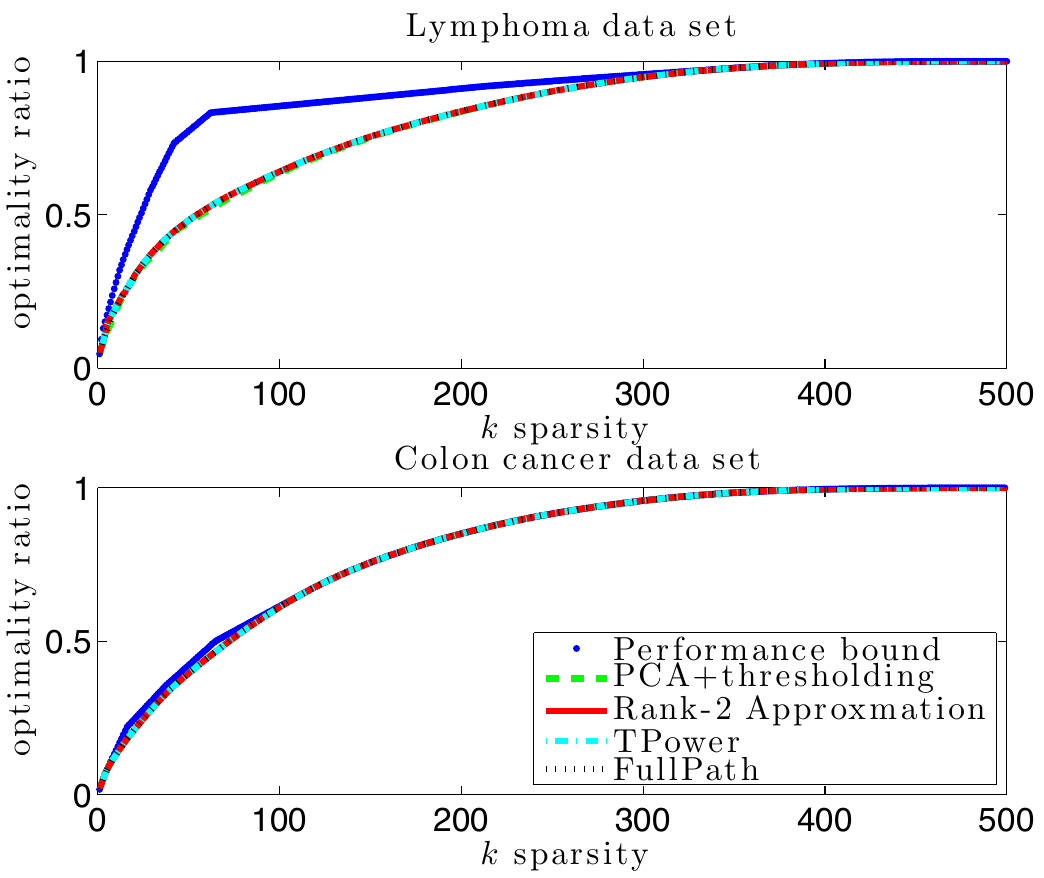}
\end{center}
\label{fig:genes}
\caption{Results on gene expression data sets.}
\end{figure}

In the same manner as in the relevant sparse PCA literature, we  evaluate our approximation on two gene expression data-sets used in \citep{daspremont2007full, d2008optimal, yuan2011truncated}.
We plot the ratio of the explained variance coming from the first sparse PC to the explained variance of the first eigenvector (which is equal to the first eigenvalue).
We also plot the performance outer bound derived in \citep{ d2008optimal}.
We observe that our approximation follows the same optimality pattern as most previous methods, for many values of sparsity $k$.
In these experiments we did not test the GPower method since the output sparsity cannot be explicitly predetermined. However, previous literature indicates that GPower is also near-optimal in this scenario.

\begin{table*}[pht]
\vspace{-1cm}
\begin{center}
{\scriptsize
\begin{tabular}{|c|c|c|c|c|}
\hline
\multicolumn{5}{|c|}{{\bf 1st sparse PC}}\\
\hline
{\bf Rank-$1$}& {\bf TPower} &{\bf Rank-$2$}& {\bf Rank-$3$} & {\bf FullPath} \\
\hline
{skype}  &      \st{\bf eurovision}&  skype     &  skype     &  eurovision \\
{microsoft} &   skype   &    microsoft&   microsoft & finalG     \\
{billion}     & microsoft  & billion    & acquisitionG    &greeceG     \\
{acquisitionG}  &  billion  &   acquisitionG  &  billion    & greece     \\
\st{\bf eurovision} & acquisitionG   & acquiredG&  acquiredG  &lucasG     \\
{acquiredG}&  buying  &    acquiresG &acquiresG &semifinalG\\ 
{acquiresG} &acquiredG  &buying   &   buying   &   final      \\
{buying}      &acquiresG &dollarsG    & dollarsG     &contest    \\
{google}     &dollarsG &    acquisition &acquisition &stereo     \\
{dollarsG}     &acquisition& google &     google&      watching  \\
\hline
\multicolumn{5}{|c|}{performance = $\frac{\text{explained variance}}{\text{maximum explained variance}}=\frac{x^T_1Ax_1}{\lambda_1}$}\\
%&\multicolumn{1}{c}{performance=} & \multicolumn{1}{c}{}& \multicolumn{1}{c}{}&\multicolumn{1}{c|}{} \\
%& & $\frac{x^T_1Ax_1}{\lambda_1}$ & & \\
\hline
 0.9863 &   0.9861 &    {\bf 0.9870} &   {\bf  0.9870} &   0.9283\\
\hline
\hline
\multicolumn{5}{|c|}{{\bf 2nd sparse PC}}\\
\hline
{\bf Rank-$1$}& {\bf TPower} &{\bf Rank-$2$}& {\bf Rank-$3$} & {\bf FullPath} \\
\hline
greece & greece & eurovision &eurovision &skype       \\
greeceG&  greeceG&  greece &    greece    & microsoft   \\
love  &  love  &  greeceG  &   lucasG    & billion     \\
lucasG&  loukas & finalG   &finalG   & acquisitionG  \\  
final   &finalsG  &lucasG     &final  &    acquiresG \\
greek&   \st{\bf athens} & final   &   stereo  &   acquiredG  \\
\st{\bf athens} & final &  stereo &    semifinalG &buying      \\
finalG&  stereo & semifinalG& contest    &dollarsG     \\
stereo & country &contest   & greeceG  &   official    \\
country & \st{\bf sailing} &songG &  watching   &google\\
%\hline
%& & var./OPT & & \\
\hline
\multicolumn{5}{|c|}{performance = $\frac{\text{explained variance}}{\text{maximum explained variance}}=\frac{\sum_{i=1}^2x^T_iAx_i}{\sum_{i=1}^2\lambda_i}$}\\
\hline
  0.8851 &    0.8850    &0.9850   & {\bf 0.9852}   & {\bf 0.9852}\\
\hline
\hline
\multicolumn{5}{|c|}{{\bf 3rd sparse PC}}\\
\hline
{\bf Rank-$1$}& {\bf TPower} &{\bf Rank-$2$}& {\bf Rank-$3$} & {\bf FullPath} \\
\hline
\st{\bf downtown}G &   \st{\bf twitter} &  love&     love&     love      \\
censusG  &censusG  & received &received &received  \\
\st{\bf athensG}  &  homeG&     greek &   twitter & damon     \\
homeG  &   \st{\bf google}&    \st{\bf know}  &   \st{\bf know} &    greek     \\
\st{\bf twitter} & yearG&   damon &   greek   & hate      \\
yearG &   greek    & amazing & damon  &  \st{\bf know}    \\  
\st{\bf murderG }&mayG  &   hate    & hate  &   amazing   \\
\st{\bf songG} & \st{\bf facebook} & twitter  &amazing&  sweet     \\
mayG &     startsG   &great &   great  & great     \\
yearsG  &  populationG &sweet    &sweet   & \st{\bf songs} \\    
\hline
\multicolumn{5}{|c|}{performance = $\frac{\text{explained variance}}{\text{maximum explained variance}}=\frac{\sum_{i=1}^3x^T_iAx_i}{\sum_{i=1}^3\lambda_i}$}\\
\hline
  0.7875   & 0.7877  &  0.8993  & {\bf  0.8994} & {\bf   0.8994}\\
\hline
\hline
\multicolumn{5}{|c|}{{\bf 4th sparse PC}}\\
\hline
{\bf Rank-$1$}& {\bf TPower} &{\bf Rank-$2$}& {\bf Rank-$3$} & {\bf FullPath} \\
\hline
thanouG&    downtownG    &      downtownG   &  downtownG    & \st{\bf twitter}   \\
kenterisG & athensG    &      athensG    & athensG    & \st{\bf facebook}\\  
guiltyG &  yearG     &      murderG & murderG  &\st{\bf welcome} \\  
kenteris & year'sG      &    yearsG   &  yearsG    & \st{\bf account}   \\
tzekosG  & murderG   &    brutalG     & brutalG      &\st{\bf goodG}     \\
monthsG   & cameraG      &    stabbedG& stabbedG &\st{\bf followers} \\
tzekos    &crimeG & bad\_eventsG&  bad\_eventsG & { censusG}  \\
\st{\bf facebook} &crime   &      yearG   &   cameraG     & { populationG} \\
imprisonmentG &stabbedG   &   turmoilG  &   yearG      &{ homeG}     \\
penaltiesG   &brutalG      &  cameraG   &  crimeG   & { startsG}  \\
\hline
\multicolumn{5}{|c|}{performance = $\frac{\text{explained variance}}{\text{maximum explained variance}}=\frac{\sum_{i=1}^4x^T_iAx_i}{\sum_{i=1}^4\lambda_i}$}\\
\hline
   0.7174 &    0.7520  &  0.8419 &   {\bf 0.8420}  &  0.8412\\
\hline
\hline
\multicolumn{5}{|c|}{{\bf 5th sparse PC}}\\
\hline
{\bf Rank-$1$}& {\bf TPower} &{\bf Rank-$2$}& {\bf Rank-$3$} & {\bf FullPath} \\
\hline
 bravoG&   songG   &censusG  &censusG   &\st{\bf yearG}     \\
loukaG &   bravoG    & homeG    & homeG  &    \st{\bf this\_yearG}      \\
\st{\bf athensG} &   \st{\bf endG}    & populationG &populationG  &loveG\\      
\st{\bf endG}   & loukaG     & may'sG   &  may'sG     & \st{\bf birthdayG}  \\ 
\st{\bf woman}G  &likedG  &    beginsG  & beginsG    &\st{\bf i\_wishG} \\   
successG &niceG    &  generalG  &  general   &  songG   \\
niceG    &greekG   &nightG    & begunG & titleG      \\
\st{\bf youtube}&  titleG     & \st{\bf noneG}   &  comesG   &  \st{\bf memoriesG} \\
\st{\bf was\_goingG} & \st{\bf trials}G    &  yearG  &  census\_employeeG& \st{\bf trialsG}      \\
\st{\bf murderedG} &\st{\bf memoriesG} &countryG  &    yearG  &   likedG   \\
\hline
\multicolumn{5}{|c|}{performance = $\frac{\text{explained variance}}{\text{maximum explained variance}}=\frac{\sum_{i=1}^5x^T_iAx_i}{\sum_{i=1}^5\lambda_i}$}\\
\hline
   0.6933 &    0.7464  &  0.8343 &   {\bf 0.8345}  &  0.8241 \\
\hline
\end{tabular}
}
\end{center}
\caption{
The first 5 sparse PCs for a data-set consisting of   $65$k Tweets and $64$k 
unique words. Words that appear with a G are translated from Greek.
}
\label{tbl:eigentweets}
\end{table*}

\subsection{Large-scale Twitter data-set}

We proceed our experimental evaluation of our algorithm by testing it on a large-scale data set.
Our data-set comprises of millions of tweets coming from Greek Twitter users.
Each tweet corresponds to a list of words and has a character limit of 140 per tweet.
Although each tweet was associated with metadata, such us hyperlinks, user id, hash tags etc., we strip these features out and just use the word list.
We use a simple Python script to normalize each Tweet.
Words that are not contextual 
 (\texttt{me, to, what,} etc) 
are discarded in an ad-hoc way.
We also discard all words that are less than three characters, or words that appear once in the corpus.
We represent each tweet as a long vector consisting of $n$ words, with a $1$ whenever a word appears, and $0$ if it does not appear.
Further details about our data set can be found in the Appendix.

Document-term data sets have been observed to follow power-laws on their eigenvalues.
Empirical results have been reported that indicate power-law like decays for eigenvalues where no  cutoff is observed \citep{dhillon2001concept} and some derived power-law generative models for 0/1 matrices  \citep{mihail2002eigenvalue, chung2003eigenvalues}.
In our experiments, we also observe power-law decays on the spectrum of the twitter matrices.
Further experimental observations of power laws can be found in the Appendix.
These underlying decay laws on the spectrum were sufficient to give good approximation guarantees; for many of our data sets $1-\epsilon$ was between $0.5$ to $0.7$, even for $d=2,3$.
Further, our algorithm empirically performed  better than these guarantees.

In the following tests, we compare against TPower and FullPath.
TPower is run for 10$k$ iterations, and is initialized with a vector having $1$s on the $k$ words of highest variance.
For FullPath we restrict the covariance to its first $5$k words of highest variance, since for larger numbers the algorithm became slow to test on a personal desktop computer.
In our experiments, we use a simpler deflation method, than the more sophisticated ones used before.
Once $k$ words appear in the first $k$-sparse PC, we strip them from the data set, recompute the new convariance, and then run all algorithms.
A benefit of this deflation is that it { forces} all sparse PCs to be orthogonal to each other
which helps for a more fair comparison with respect to explained variance.
Moreover, this deflation preserves the sparsity of the matrix $A$ after each deflation step; sparsity on $A$ facilitates faster execution times for all methods tested.
The performance metric here is again the explained variance over its maximum possible value:
if we compute $L$ PCs, $x_1,\ldots, x_L$, we measure their performance as $\frac{\sum_{i=1}^Lx_i^TAx_i}{\sum_{i=1}^L\lambda_i}$.
We see that in many experiments, we come very close to the optimal value of $1$.

In Table~\ref{tbl:tweet_results}, we show our results
for all tweets that contain the word \texttt{Japan}, 
for a 5-day (May 1-5, 2011) and then a month-length time window (May, 2011).
In all these tests, our rank-3 approximation consistently captured more variance than all other compared methods.

In Table~\ref{tbl:eigentweets}, we show a day-length experiment (May 10th, 2011), where we had $65$k Tweets and $64$k unique words.
For this data-set we report the first $5$ sparse PCs generated by all methods tested.
The average computation times for this time-window where less than 1 second for the rank-1 approximation, less than 5 seconds for rank-2,  and  less than 2 minutes for the rank-3 approximation on a Macbook Pro 5.1 running MATLAB 7.
The main reason for these tractable running times is the use of our elimination scheme which left only around $40-80$ rows of the initial matrix of $64$k rows.
In terms of running speed, we empirically observed that our algorithm is slower than Tpower but faster than FullPath for the values of $d$ tested.
In Table~\ref{tbl:eigentweets}, words with strike-through are what we consider non-matching to the ``main topic'' of that PC.
Words marked with G are translated from Greek.
From the PCs we see that the main topics are about Skype's acquisition by Microsoft, the European Music Contest ``Eurovision,''  a crime that occurred in the downtown of Athens, and the Greek census that was carried for the year 2011.
An interesting observation is that a general ``excitement" sparse principal component appeared in most of our queries on the Twitter data set. 
It involves words like \texttt{like, love, liked, received, great}, etc, and was generated by all algorithms.

\section{Conclusions}

We conclude that our algorithm can efficiently provide interpretable sparse PCs while matching or outperforming the accuracy of previous methods.
A parallel implementation in the MapReduce framework and larger data studies are very interesting future directions.

\begin{table}[H]
 \centerline{
{
\begin{tabular}{c|c|c|c}
\hline
\hline
& *japan & 1-5 May 2011 & May 2011\\
\hline
$m\times n$& $12\text{k}\times 15\text{k}$ & $267\text{k}\times 148\text{k}$ & $1.9\text{mil}\times222\text{k}$\\
% & $n=15$k  & $n=148$k & $n=222$k\\
% $k$, \#PCs & $k =10$, $L=5$ & $k =4$, $L=7$ & $k = 5$, $L=3$ \\
$k$  & $k =10$ & $k =4$ & $k = 5$ \\
\#PCs & $5$  & $7$ & $3$ \\
\hline
\hline
Rank-1& $0.600$&$0.815$ & $0.885$\\
\hline
TPower&  $0.595$ &$0.869$ & $ 0.915$\\
\hline
Rank-2  & ${\bf 0.940}$  &$0.934$ & $0.885$\\
\hline
Rank-3  & ${\bf 0.940}$  & ${\bf0.936}$ & ${\bf 0.954}$\\
\hline
FullPath  &  $0.935$ &$0.886$ & $0.953$\\
\hline
\end{tabular}
}
}
\caption{Performance comparison on the Twitter data-set}
\label{tbl:tweet_results}
\end{table}

\bibliography{lowrankSPCA_long}

\begin{thebibliography}{33}
\providecommand{\natexlab}[1]{#1}
\providecommand{\url}[1]{\texttt{#1}}
\expandafter\ifx\csname urlstyle\endcsname\relax
  \providecommand{\doi}[1]{doi: #1}\else
  \providecommand{\doi}{doi: \begingroup \urlstyle{rm}\Url}\fi

\bibitem[Amini and Wainwright(2008)]{amini2008high}
A.A. Amini and M.J. Wainwright.
\newblock High-dimensional analysis of semidefinite relaxations for sparse
  principal components.
\newblock In \emph{Information Theory, 2008. ISIT 2008. IEEE International
  Symposium on}, pages 2454--2458. IEEE, 2008.

\bibitem[Asteris et~al.(2011)Asteris, Papailiopoulos, and
  Karystinos]{asteris2011sparse}
M.~Asteris, D.S. Papailiopoulos, and G.N. Karystinos.
\newblock Sparse principal component of a rank-deficient matrix.
\newblock In \emph{Information Theory Proceedings (ISIT), 2011 IEEE
  International Symposium on}, pages 673--677. IEEE, 2011.

\bibitem[Berthet and Rigollet(2012)]{berthet2012optimal}
Quentin Berthet and Philippe Rigollet.
\newblock Optimal detection of sparse principal components in high dimension.
\newblock \emph{arXiv preprint arXiv:1202.5070}, 2012.

\bibitem[Berthet and Rigollet(2013{\natexlab{a}})]{berthet2013complexity}
Quentin Berthet and Philippe Rigollet.
\newblock Complexity theoretic lower bounds for sparse principal component
  detection, 2013{\natexlab{a}}.

\bibitem[Berthet and Rigollet(2013{\natexlab{b}})]{berthet2013computational}
Quentin Berthet and Philippe Rigollet.
\newblock Computational lower bounds for sparse pca.
\newblock \emph{arXiv preprint arXiv:1304.0828}, 2013{\natexlab{b}}.

\bibitem[Cadima and Jolliffe(1995)]{cadima1995loading}
J.~Cadima and I.T. Jolliffe.
\newblock Loading and correlations in the interpretation of principle
  compenents.
\newblock \emph{Journal of Applied Statistics}, 22\penalty0 (2):\penalty0
  203--214, 1995.

\bibitem[Cai et~al.(2012)Cai, Ma, and Wu]{cai2012sparse}
T~Tony Cai, Zongming Ma, and Yihong Wu.
\newblock Sparse pca: Optimal rates and adaptive estimation.
\newblock \emph{arXiv preprint arXiv:1211.1309}, 2012.

\bibitem[Cai et~al.(2013)Cai, Ma, and Wu]{cai2013optimal}
Tony Cai, Zongming Ma, and Yihong Wu.
\newblock Optimal estimation and rank detection for sparse spiked covariance
  matrices.
\newblock \emph{arXiv preprint arXiv:1305.3235}, 2013.

\bibitem[Chung et~al.(2003)Chung, Lu, and Vu]{chung2003eigenvalues}
F.~Chung, L.~Lu, and V.~Vu.
\newblock Eigenvalues of random power law graphs.
\newblock \emph{Annals of Combinatorics}, 7\penalty0 (1):\penalty0 21--33,
  2003.

\bibitem[d'Aspremont et~al.(2007{\natexlab{a}})d'Aspremont, El~Ghaoui, Jordan,
  and Lanckriet]{d2007direct}
A.~d'Aspremont, L.~El~Ghaoui, M.I. Jordan, and G.R.G. Lanckriet.
\newblock A direct formulation for sparse pca using semidefinite programming.
\newblock \emph{SIAM review}, 49\penalty0 (3):\penalty0 434--448,
  2007{\natexlab{a}}.

\bibitem[d'Aspremont et~al.(2008)d'Aspremont, Bach, and Ghaoui]{d2008optimal}
A.~d'Aspremont, F.~Bach, and L.E. Ghaoui.
\newblock Optimal solutions for sparse principal component analysis.
\newblock \emph{The Journal of Machine Learning Research}, 9:\penalty0
  1269--1294, 2008.

\bibitem[d'Aspremont et~al.(2012)d'Aspremont, Bach, and
  Ghaoui]{d2012approximation}
A.~d'Aspremont, F.~Bach, and L.E. Ghaoui.
\newblock Approximation bounds for sparse principal component analysis.
\newblock \emph{arXiv preprint arXiv:1205.0121}, 2012.

\bibitem[d'Aspremont et~al.(2007{\natexlab{b}})d'Aspremont, Bach, and
  Ghaoui]{daspremont2007full}
Alexandre d'Aspremont, Francis~R. Bach, and Laurent~El Ghaoui.
\newblock Full regularization path for sparse principal component analysis.
\newblock In \emph{Proceedings of the 24th international conference on Machine
  learning}, ICML '07, pages 177--184, 2007{\natexlab{b}}.

\bibitem[Dhillon and Modha(2001)]{dhillon2001concept}
I.S. Dhillon and D.S. Modha.
\newblock Concept decompositions for large sparse text data using clustering.
\newblock \emph{Machine learning}, 42\penalty0 (1):\penalty0 143--175, 2001.

\bibitem[Gawalt et~al.(2010)Gawalt, Zhang, and El~Ghaoui]{gawaltsparse}
B.~Gawalt, Y.~Zhang, and L.~El~Ghaoui.
\newblock Sparse pca for text corpus summarization and exploration.
\newblock \emph{NIPS 2010 Workshop on Low-Rank Matrix Approximation}, 2010.

\bibitem[Horn and Johnson(1990)]{horn1990matrix}
R.A. Horn and C.R. Johnson.
\newblock \emph{Matrix analysis}.
\newblock Cambridge university press, 1990.

\bibitem[Jolliffe(1995)]{jolliffe1995rotation}
I.T. Jolliffe.
\newblock Rotation of principal components: choice of normalization
  constraints.
\newblock \emph{Journal of Applied Statistics}, 22\penalty0 (1):\penalty0
  29--35, 1995.

\bibitem[Jolliffe et~al.(2003)Jolliffe, Trendafilov, and
  Uddin]{jolliffe2003modified}
I.T. Jolliffe, N.T. Trendafilov, and M.~Uddin.
\newblock A modified principal component technique based on the lasso.
\newblock \emph{Journal of Computational and Graphical Statistics}, 12\penalty0
  (3):\penalty0 531--547, 2003.

\bibitem[Journ{\'e}e et~al.(2010)Journ{\'e}e, Nesterov, Richt{\'a}rik, and
  Sepulchre]{journee2010generalized}
M.~Journ{\'e}e, Y.~Nesterov, P.~Richt{\'a}rik, and R.~Sepulchre.
\newblock Generalized power method for sparse principal component analysis.
\newblock \emph{The Journal of Machine Learning Research}, 11:\penalty0
  517--553, 2010.

\bibitem[Kaiser(1958)]{kaiser1958varimax}
H.F. Kaiser.
\newblock The varimax criterion for analytic rotation in factor analysis.
\newblock \emph{Psychometrika}, 23\penalty0 (3):\penalty0 187--200, 1958.

\bibitem[Karystinos and Liavas(2010)]{karystinos2010efficient}
G.N. Karystinos and A.P. Liavas.
\newblock Efficient computation of the binary vector that maximizes a
  rank-deficient quadratic form.
\newblock \emph{Information Theory, IEEE Transactions on}, 56\penalty0
  (7):\penalty0 3581--3593, 2010.

\bibitem[Kuleshov(2013)]{kuleshov2013fast}
Volodymyr Kuleshov.
\newblock Fast algorithms for sparse principal component analysis based on
  rayleigh quotient iteration.
\newblock 2013.

\bibitem[Mackey(2009)]{mackey2009deflation}
L.~Mackey.
\newblock Deflation methods for sparse pca.
\newblock \emph{Advances in neural information processing systems},
  21:\penalty0 1017--1024, 2009.

\bibitem[Mihail and Papadimitriou(2002)]{mihail2002eigenvalue}
M.~Mihail and C.~Papadimitriou.
\newblock On the eigenvalue power law.
\newblock \emph{Randomization and approximation techniques in computer
  science}, pages 953--953, 2002.

\bibitem[Moghaddam et~al.(2006{\natexlab{a}})Moghaddam, Weiss, and
  Avidan]{moghaddam2006generalized}
B.~Moghaddam, Y.~Weiss, and S.~Avidan.
\newblock Generalized spectral bounds for sparse lda.
\newblock In \emph{Proceedings of the 23rd international conference on Machine
  learning}, pages 641--648. ACM, 2006{\natexlab{a}}.

\bibitem[Moghaddam et~al.(2006{\natexlab{b}})Moghaddam, Weiss, and
  Avidan]{moghaddam2006spectral}
B.~Moghaddam, Y.~Weiss, and S.~Avidan.
\newblock Spectral bounds for sparse pca: Exact and greedy algorithms.
\newblock \emph{Advances in neural information processing systems},
  18:\penalty0 915, 2006{\natexlab{b}}.

\bibitem[Moghaddam et~al.(2007)Moghaddam, Weiss, and Avidan]{moghaddam2007fast}
B.~Moghaddam, Y.~Weiss, and S.~Avidan.
\newblock Fast pixel/part selection with sparse eigenvectors.
\newblock In \emph{Computer Vision, 2007. ICCV 2007. IEEE 11th International
  Conference on}, pages 1--8. IEEE, 2007.

\bibitem[Shen and Huang(2008)]{shen2008sparse}
H.~Shen and J.Z. Huang.
\newblock Sparse principal component analysis via regularized low rank matrix
  approximation.
\newblock \emph{Journal of multivariate analysis}, 99\penalty0 (6):\penalty0
  1015--1034, 2008.

\bibitem[Sriperumbudur et~al.(2007)Sriperumbudur, Torres, and
  Lanckriet]{sriperumbudur2007sparse}
B.K. Sriperumbudur, D.A. Torres, and G.R.G. Lanckriet.
\newblock Sparse eigen methods by dc programming.
\newblock In \emph{Proceedings of the 24th international conference on Machine
  learning}, pages 831--838. ACM, 2007.

\bibitem[Yuan and Zhang(2011)]{yuan2011truncated}
X.T. Yuan and T.~Zhang.
\newblock Truncated power method for sparse eigenvalue problems.
\newblock \emph{arXiv preprint arXiv:1112.2679}, 2011.

\bibitem[Zhang and El~Ghaoui(2011)]{zhang2011large}
Y.~Zhang and L.~El~Ghaoui.
\newblock Large-scale sparse principal component analysis with application to
  text data.
\newblock \emph{Advances in Neural Information Processing Systems}, 2011.

\bibitem[Zhang et~al.(2012)Zhang, d'Aspremont, and Ghaoui]{zhang2012sparse}
Y.~Zhang, A.~d'Aspremont, and L.E. Ghaoui.
\newblock Sparse pca: Convex relaxations, algorithms and applications.
\newblock \emph{Handbook on Semidefinite, Conic and Polynomial Optimization},
  pages 915--940, 2012.

\bibitem[Zou et~al.(2006)Zou, Hastie, and Tibshirani]{zou2006sparse}
H.~Zou, T.~Hastie, and R.~Tibshirani.
\newblock Sparse principal component analysis.
\newblock \emph{Journal of computational and graphical statistics}, 15\penalty0
  (2):\penalty0 265--286, 2006.

\end{thebibliography}

\section*{Appendix}
\begin{appendix}
\section{Rank-$d$ candidate optimal sets $\mathcal{S}_d$}

In this section, we generalize the concepts of the $\mathcal{S}_2$ construction to the general $d$ case and prove the following:
\begin{lem} (\cite{asteris2011sparse})
The rank-$d$ optimal set $\mathcal{S}_d$ has $O(n^d)$ candidate optimal solutions and can be build in time $O(n^{d+1}\log n)$.
\label{rankd_lemma}
\end{lem}

Here, we use $A_d= V_d V_d^T$
where $ V_d=\left[\sqrt{\lambda_1}\cdot v_1 \;\;\ldots \;\;\sqrt{\lambda_d}\cdot v_d\right]$.
We can rewrite the optimization on $A_d$ as
 \begin{align}
  \max_{x\in \mathbb{S}_k} x^TA_dx=\max_{x\in \mathbb{S}_k} {\left\| V^T_dx\right\|_2^2}
=\max_{x\in\mathbb{S}_k,\| c\|_2=1}\left( c^T V^T_dx\right)^2=\max_{x\in\mathbb{S}_k,\| c\|_2=1}\left( v_{ c}^Tx\right)^2,
 \label{cVxd}
\end{align}
where $ v_ c= V_d c$.
Again, for a fixed unit vector $ c$, $\mathcal{I}_k( v_{ c})$ is the locally optimal support set of the $k$-sparse  vector that maximizes $\left( v_{ c}^Tx\right)^2$.

{\bf Hyperspherical variables and intersections.} 
For this case, we introduce $d-1$ angles $\varphi=[\phi_1,\ldots, \phi_{d-1}] \in \left(-\frac{\pi}{2},\; \frac{\pi}{2}\right]^{d-1}$ which are used to restate $ c$ as
a hyperspherical unit vector
\begin{equation}
 c =\left[
\begin{array}{c}
\sin\phi_1\\
 \cos\phi_1\sin\phi_2\\
\vdots \\
\cos\phi_1\cos\phi_2\hdots\sin\phi_{d-1}\\
\cos\phi_1\cos\phi_2\hdots\cos\phi_{d-1}
\end{array}
\right].\nonumber
\end{equation}
In this case, an element of $ v_ c =  V_d c$ is a continuous $d$-dimensional function on $d-1$ variables $\varphi=[\phi_1,\ldots, \phi_{d-1}]$, i.e., it is a $(d-1)$-dimensional hypersurface:
\begin{align}
[ v(\varphi)]_i &= \sin\phi_1 \cdot [\sqrt{\lambda_1} v_1]_i\nonumber
+ \cos\phi_1\sin\phi_2\cdot [\sqrt{\lambda_2} v_2]_i+\ldots \nonumber
+\cos\phi_1\cos\phi_2\hdots\cos\phi_{d-1} [\sqrt{\lambda_d} v_d]_i\nonumber.
\end{align}
In Fig. \ref{fig:rank3}, we draw a $d=3$ example spannogram of $4$ curves, from a randomly generated matrix $ V_3$.
 \begin{figure}[h]
\centerline{
 \includegraphics[width=0.87\columnwidth]{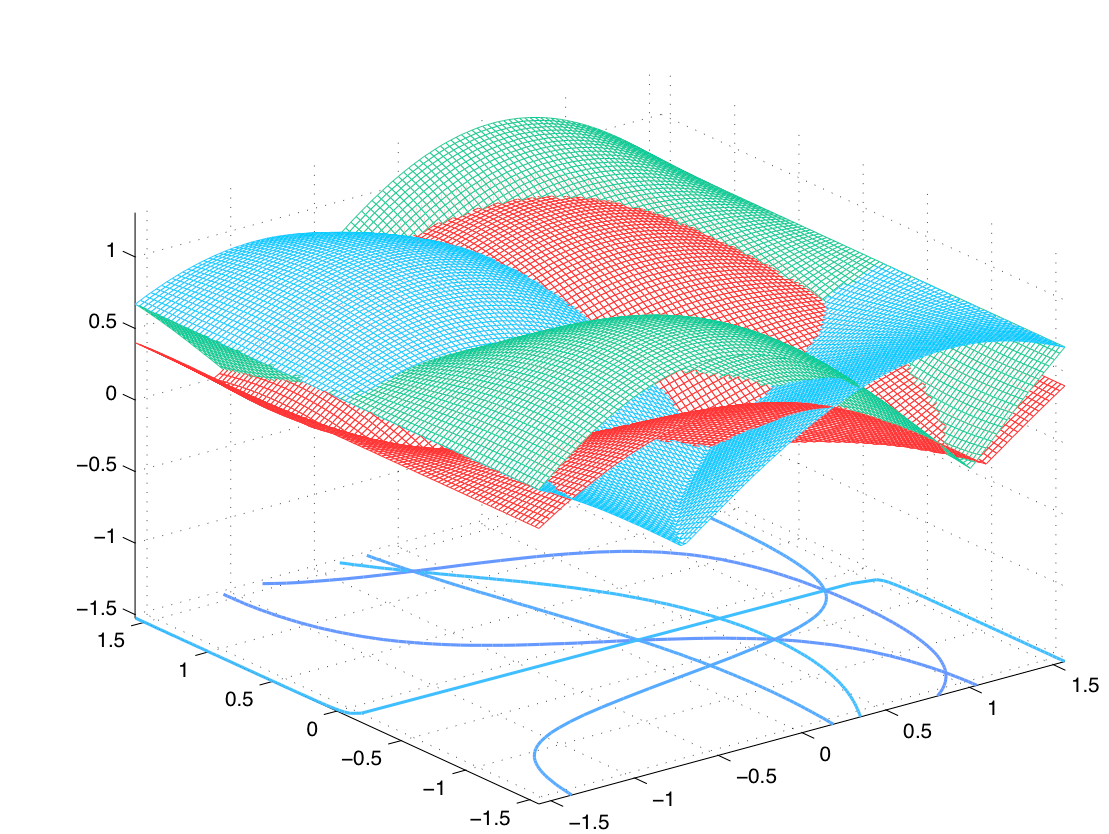}
}
\caption{A rank-$3$ spannogram for a $3\times 3$ matrix ${\bf V}_3$.}
\label{fig:rank3}
\end{figure}

Calculating $\mathcal{I}_k( v_ c)$ for a fixed vector $ c$ is equivalent to finding the relative sortings of the $n$ (absolute value) hypersurfaces $|[ v(\varphi)]_i| $ at that point. 
The relative sorting between two surfaces $[ v(\varphi)]_i$ and $[ v(\varphi)]_j$ changes on the point where $[ v(\varphi)]_i= [ v(\varphi)]_j$.
A difference with the rank-$2$ case, is that here, the solution to the equation $[ v(\varphi)]_i= [ v(\varphi)]_j$ is not a single point (i.e., a single solution vector $ c$), but a $(d-1)$ dimensional space of solutions.
Let 
\begin{equation}
\mathcal{X}_{i,j} = \left\{ v(\varphi): [ v(\varphi)]_i= [ v(\varphi)]_j, \varphi\in \left(-\frac{\pi}{2},\; \frac{\pi}{2}\right]^{d-1}\right\}\nonumber
\end{equation}
be the the set of all $ v(\varphi)$ vectors where $[ v(\varphi)]_i= [ v(\varphi)]_j$ in the $\varphi$ domain.
Then, the elements of the set $\mathcal{X}_{i,j}$ correspond exactly to the intersection points between hypersurfaces $i$ and $j$.

Since locally optimal support sets change only with the local sorting changes, the intersection points defined by all $\mathcal{X}_{i,j}$ sets are the only points of interest. 

{\bf Establishing all intersection vectors.}
We will now work on the $(d-2)$-dimensional space $\mathcal{X}_{i,j}$.
For the vectors in this space, there are again critical $\varphi$ points, where both the $i$ and $j$ coordinates become members of a top-$k$ set.
These are the points when locally optimal support sets change.
This happens when the $i$ and $j$ coordinates become equal with another $l$-th coordinate of $ v(\phi)$.
This new space of $ v(\phi)$ vectors where coordinates $i$,$j$,and $l$ are equal will be the set $\mathcal{X}_{i,j,k}$, this will now be a $(d-3)$-dimensional subspace.
These intersection points can be used to generate all locally optimal support sets.
From the previous set we need to only check points where the three curves studied intersect with an additional one that is the bottom curve of a top $k$ set.
Again, these intersection points are sufficient to generate all locally optimal support sets.
We can proceed in that manner until we reach the single-element set $\mathcal{X}_{i_1,i_2,\ldots, i_d}$ which corresponds to the vector $ v(\phi)$ defined over $\Phi^{d-1}$, where $d$ curves intersect, i.e.,
\begin{equation}
[ v(\varphi)]_{i_1} = [ v(\varphi)]_{i_2}=\ldots = [ v(\varphi)]_{i_d}.
\end{equation}
Observe that for each of these $d$ curves we need to check both of their signs, that is all equations $[ v(\varphi)]_{i_1} =b_1[ v(\varphi)]_{i_2}=\ldots = b_{d-1}[ v(\varphi)]_{i_d}$ need to be solved for $c$, where $b_i\in \{\pm1\}$.
Therefore, the total number of intersection vectors $ c$ is equal to 
$2^{d-1}{n\choose d}$.
These are the only vectors that need to be checked, since they are the potential ones where $d$-tuples of curves may become part of the top $k$ support set.

\begin{algorithm}[h]
    \caption{Spannogram Algorithm for $\mathcal{S}_d$.}
\begin{algorithmic}[1]
   \STATE {\bfseries Input:} $k$, $p$, $ V_d = \left[\sqrt{\lambda_{1}} v_1\ldots \sqrt{\lambda_{1}} v_1\right]$
\STATE Initialize $\mathcal{S}_d\leftarrow\emptyset$
\STATE $\mathcal{B}\leftarrow \{1,\ldots, 1\}$ 
\IF{$p = 0$ }
% \ELSE %p = 0 // $A$ has mixed entries
\STATE $\mathcal{B} \leftarrow \{b_1,\ldots, b_{d-1}\}\in\{\pm 1\}^{d-1}$ 
\ENDIF
\FOR{ all $n\choose d$ subsets $(i_1,\ldots, i_d)$ from $\{1,\ldots, n\}$} 
\FOR{ all sequences $(b_1,\ldots, b_{d-1})\in\mathcal{B}$} 
   \STATE
%  Calculate \\
$ c \leftarrow \footnotesize{ \text{nullspace}\left(
\left[
\begin{array}{c}
 e^T_{i_1}-b_1\cdot e^T_{i_2}\\
\vdots\\
 e^T_{i_1}-b_{d-1}\cdot e^T_{i_d}\\
\end{array}
\right] V_d\right)}$

\IF{$p = 1$}
\STATE $\mathcal{I}\leftarrow\text{indices of the $k$-top elements of }  V c$
\ELSE
\STATE $\mathcal{I}\leftarrow\text{indices of the $k$-top elements of } \text{abs}( V c)$
\ENDIF
\STATE $l\leftarrow 1$
\STATE $\mathcal{J}_1 \leftarrow\mathcal{I}_{1:k}$
\STATE $r\leftarrow\left|\mathcal{J}_1\cap (i_1,\ldots, i_d)\right|$
\IF {$r<d$}
\FOR{all $r$-subsets $\mathcal{M}$ from $(i_1,\ldots, i_d)$}
\STATE$l\leftarrow l+1$
\STATE$\mathcal{J}_l\leftarrow \mathcal{I}_{1:k-r}\cup\mathcal{M}$
\ENDFOR
\ENDIF
\STATE $\mathcal{S}_d \leftarrow  \mathcal{S}_d\cup  \mathcal{J}_1\ldots \cup \mathcal{J}_l$.
\ENDFOR
 \ENDFOR
 \STATE {\bf Output:} $\mathcal{S}_d$.
\end{algorithmic}
\end{algorithm}

{\bf Building $\mathcal{S}_d$.}
To visit all possible sorting candidates, we now have to check the intersection points, 
which are obtained by solving the system of $d-1$, linear in $ c(\phi_{1:d-1})$, equations
\begin{align}
&[ v(\varphi)]_{i_1} =b_1[ v(\varphi)]_{i_2}=\ldots = b_{d-1}[ v(\varphi)]_{i_d}\nonumber\\
\Leftrightarrow
&[ v(\varphi)]_{i_1} = b_{1}[ v(\varphi)]_{i_2}, \ldots , [ v(\varphi)]_{i_1}= b_{d-1}[ v(\varphi)]_{i_d}
\label{rankdEq}
\end{align}
where $b_i\in \{\pm1\}$.
Then, we can rewrite the above as
\begin{equation}
\left[
\begin{array}{c}
 e^T_{i_1}-b_1 e^T_{i_2}\\
\vdots\\
 e^T_{i_1}-b_{d-1} e^T_{i_d}
\end{array}
\right] V c = 0_{(d-1)\times 1} 
\label{eq:rankdeq}
\end{equation}
where the solution is the nullspace of the matrix multiplying $c$, which has dimension 1.

To explore all possible candidate vectors, we need to compute all possible $2^{d-1}{n \choose d}$ solution intersection vectors $ c$.
On each intersection vector we need to compute the locally optimal support set $\mathcal{I}_k( v_ c)$.
Then, due to the fact that the $i_1,\ldots, i_d$ coordinates of $ v_ c$ have the same absolute value, we need to compute, on $ c$, at most ${d \choose \left\lceil\frac{d}{2}\right\rceil}$ tuples of elements that are potentially in the boundary of the bottom $k$ elements. 
This is done to secure that all candidate support sets ``neighboring'' an intersection point are included in $\mathcal{S}_d$.
This, induces at most ${d \choose \left\lceil\frac{d}{2}\right\rceil}$ local sortings, i.e., rank-$1$ instances.
All these sorting will eventually be the elements of the $\mathcal{S}_d$ set.
The number of all candidate support sets will now be $2^{d-1}{d \choose \left\lceil\frac{d}{2}\right\rceil}{n\choose d}$ and the total computation complexity is $O\left(n^{d+1}\log n\right)$.

\section{Nonnegative matrix speed-up}
In this section we show that if $A$ has nonnegative entries then we can speed up computations by a factor of $2^{d-1}$.
The main idea behind this speed-up is that when $A$ has only nonnegative entries, then in our intersection equations in Eq.~(\ref{eq:rankdeq}) we do not need to check all possible signed combinations of the $d$ curves.
In the following we explain this point.

We first note that the Perron-Frobenius theorem \citep{horn1990matrix} grants us the fact that the optimal solution $x_*$ will have nonnegative entries.
That is, if $A$ has nonnegative entries, then $x^*$ will also have nonnegative entries.
This allows us to pose a redundant nonnegativity constraint on our optimization
\begin{equation}
\max_{x\in \mathbb{S}_k} x^TAx=\max_{x\in \mathbb{S}_k, x\succeq 0} x^TAx.
\end{equation}
 Our approximation uses the above constraint to reduce the cardinality of $\mathcal{S}_d$ by a factor of $2^{d-1}$.
Let us consider for example the rank 1 case:
 \begin{equation}
\max_{x\in \mathbb{S}_{k}, x\succeq 0}\left( v^Tx\right)^2 = \max_{x\in \mathbb{S}_{k}, x\succeq 0}\left(\sum_{i=1}^nv_i x_i\right)^2\nonumber
% \label{rank1}
 \end{equation}
Here,the optimal solution can be again found in time $\mathcal{O}(n\log n)$. 
First, we sort the elements of $ v$. 
The optimal support $\mathcal{I}$ the for above problem corresponds to either the top $k$, or the bottom $k$ unsigned elements of the sorted $ v$.
The fact that is important here is that  the optimal vector can only have entries of the same sign.\footnote{If there are less than $k$ elements of the same sign in either of the two support sets in $\mathcal{I}_1$, then, and in order to satisfy the sparsity constraint, we can put weight $\epsilon>0$ on elements with the least amplitude in such set and opposite sign.
This will only perturb the objective by a component proportional to $\epsilon$, which can then be driven arbitrarily close to $0$, while respecting the sparsity constraint.}
The implication of the previous fact is that on our curve intersection points, we can only account for intersections of the sort $[v(\varphi)]_i = [v(\varphi)]_j$.
Intersection of the form $[v(\varphi)]_i = -[v(\varphi)]_j$ are not to be considered due to the fact that the locally optimal vector can only have one of the two signs.
This means that in Eq.~(\ref{eq:rankdeq}), we only have a single sign pattern.
This eliminates exactly a factor of $2^{d-1}$ from the cardinality of the $\mathcal{S}_d$ set.

\section{Feature Elimination}
\label{sec:elim}

\begin{figure}[t]
\begin{center}
\includegraphics[scale=0.4]{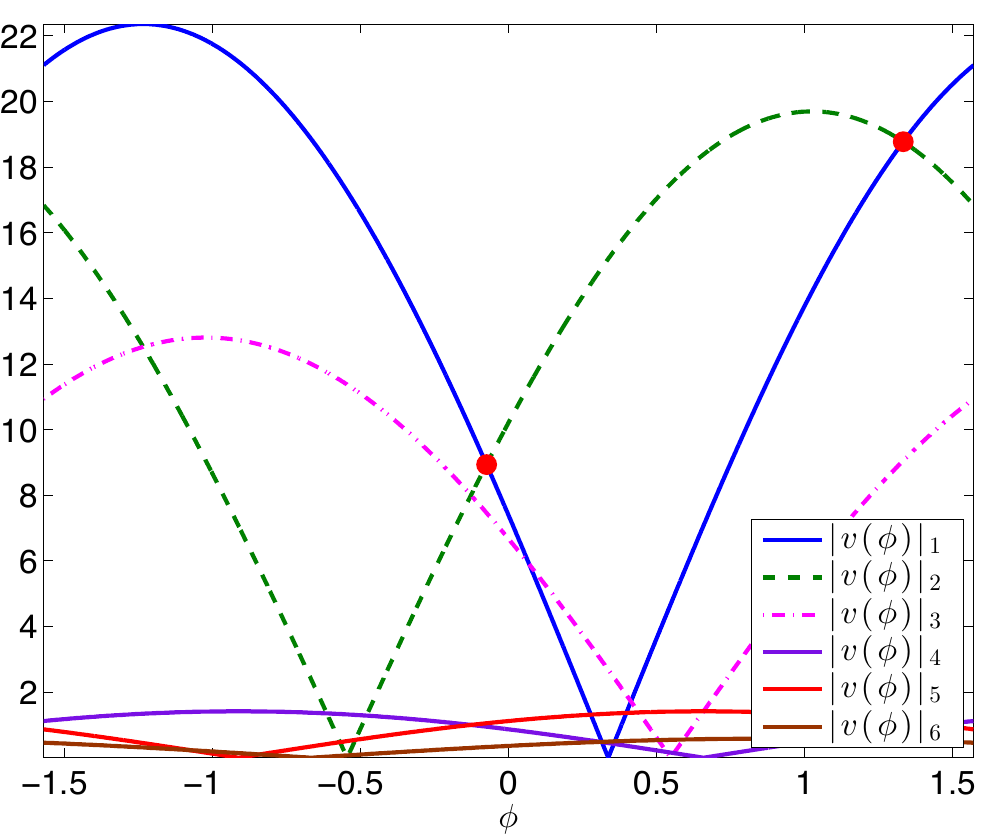}
\end{center}
\caption{An example of a spannogram for $n=6$, $d=2$. Assume that $k=1$. Then, the candidate optimal supports are $\mathcal{S}_2=\{\{1\},\{2\}\}$, that is either the blue curve ($i=1$) is the top one, or the green curve ($i=2$) becomes the top one, depending on the different values of $\phi$.
Finding the intersection points between these two curves is sufficient to recover these optimal supports.
The idea of the elimination is that curves with (maximum) amplitude less than the amplitude of these types of intersection points can be safely discarded. 
In our example, after considering the blue and green curves and obtaining their intersection points, we can see that all other curves apart from the purple curve can be discarded; their amplitudes are less than the lowest intersection point of the blue and green curves.
Our elimination step formalizes this idea.}
\label{fig:curves}
\end{figure}

\begin{figure}[bthp]
\begin{center}
   \subfigure[Start with the curves of highest amplitude. Then, find their intersection points (red dots) and put them in the set $\mathcal{P}_1$.]{\includegraphics[scale=0.4]{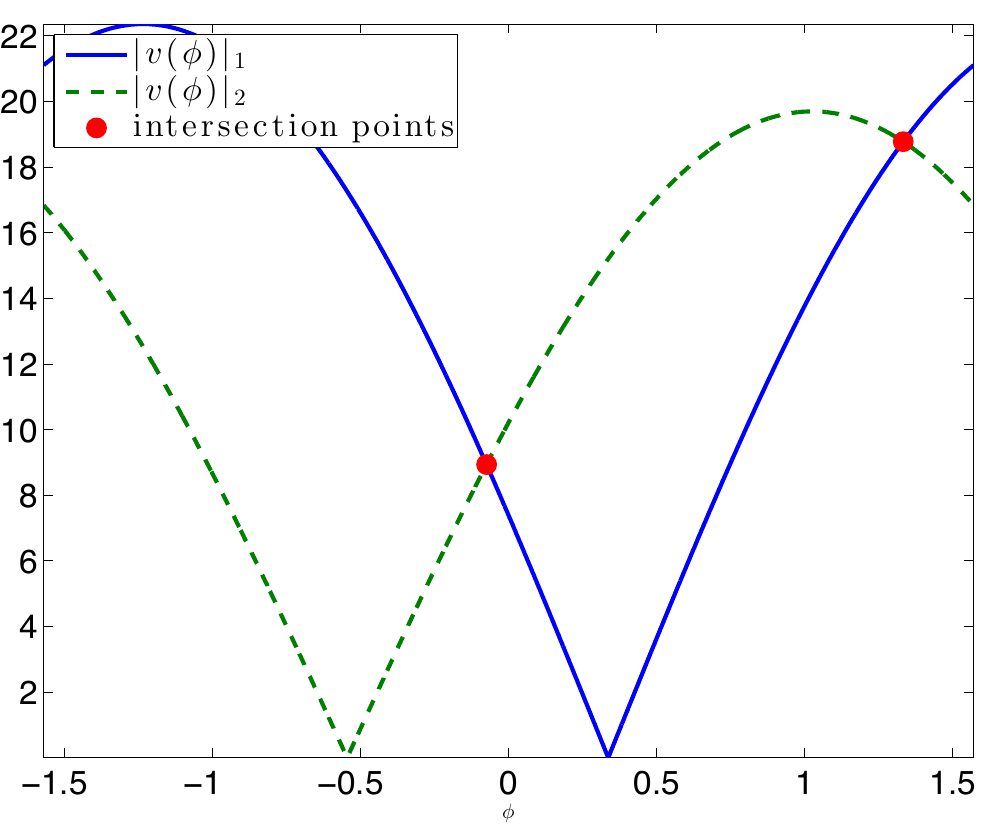}}\quad
   \subfigure[Examine the curve with amplitude that is the largest among the ones not tested yet. 
   If the curve has amplitude less than the minimum intersection point in $\mathcal{P}_1$, discard it. Also, discard all curves with amplitude less than that.
  If it has amplitude higher than the minimum point in $\mathcal{P}_1$, then compute its intersection points with the curves already examined. 
  For each new intersection point check whether it has $k-1$ curves above it. If yes, add it to $\mathcal{P}_1$. Retest all points in $\mathcal{P}_1$; if there is a point that has more than $k-1$ curves above it, discard it from $\mathcal{P}_1$.]{\includegraphics[scale=0.4]{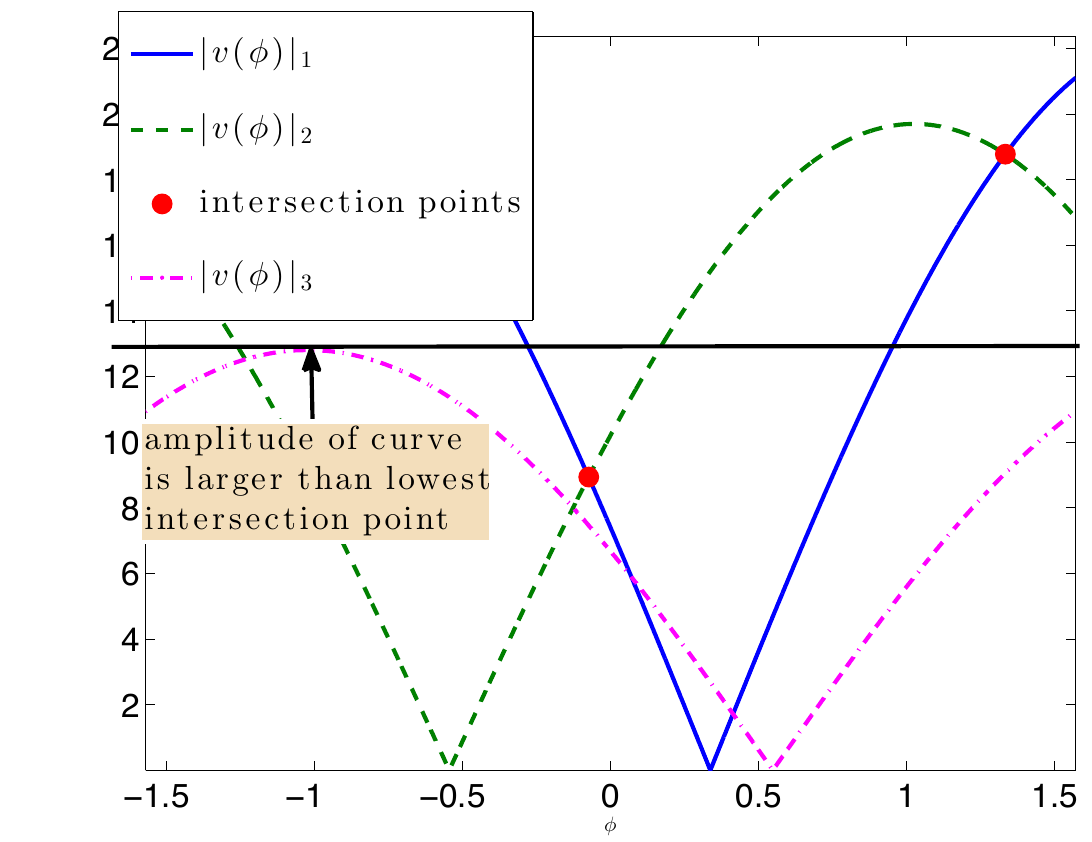}}
   \end{center}
   \begin{center}
      \subfigure[Repeat the same steps. Check if the amplitude of the lowest intersection point is higher than the amplitude of the curve next in the sorted list.]{\includegraphics[scale=0.4]{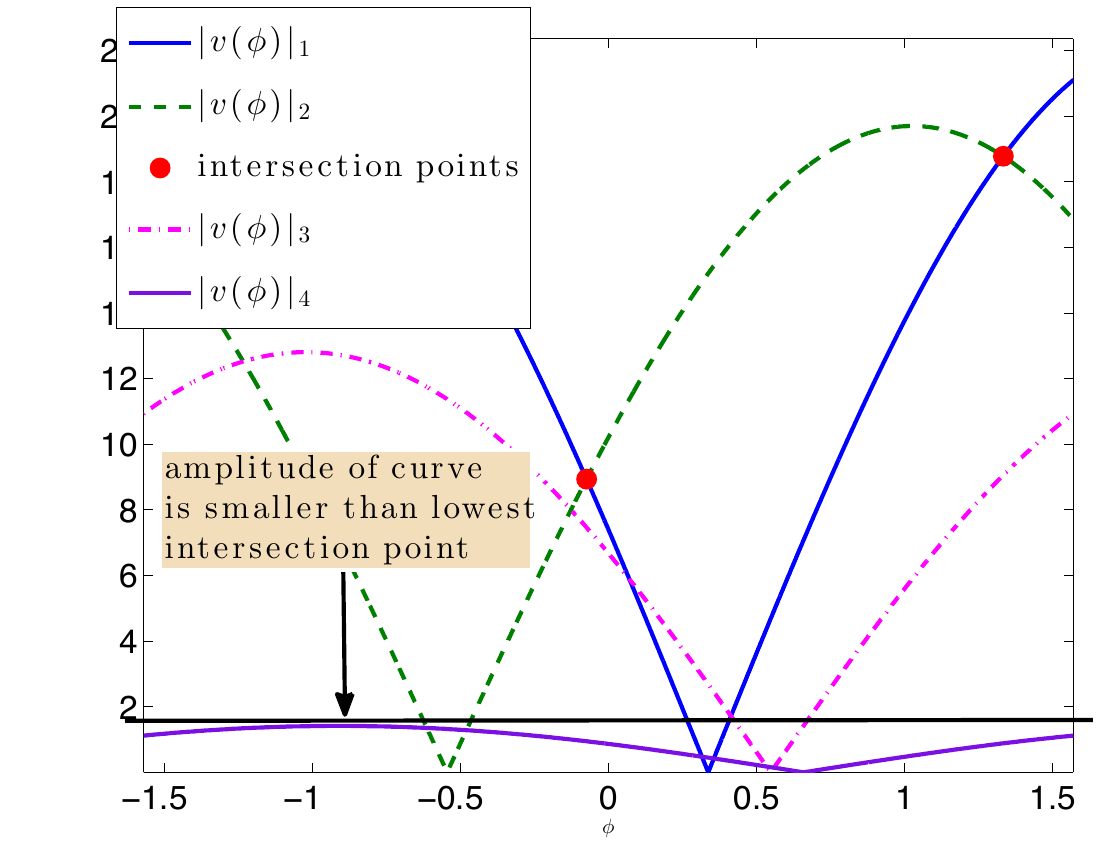}}\quad
         \subfigure[Eventually this process will end by finding a curve with amplitude less than the intersection points in $\mathcal{P}_1$. It will then discard all curves with amplitude less, as shown in the figure above.]{\includegraphics[scale=0.4]{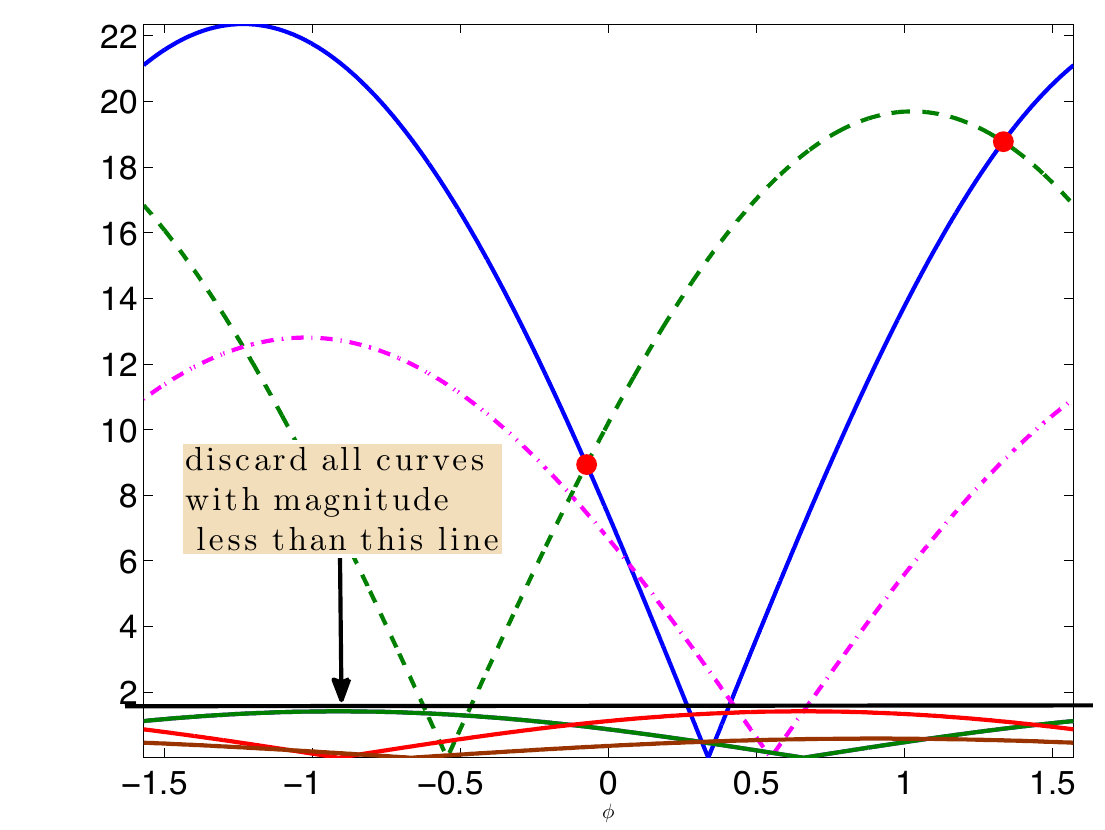}}
            \end{center}
\caption{A simple elimination example for $n=6$, $d=2$, and $k=1$.}
\label{fig:elimination_example}
\end{figure}

In this section we present our feature elimination algorithm.
This step reduces the dimension $n$ of the problem and this reduction in practice is empirically shown to be significant and allows us to run our algorithm for very large matrices $A$.
Our elimination algorithm is combinatorial and is based on sequentially checking the rows of $V_d$, depending on the value of their norm.
This step is based again on the spannogram framework used in our approximation algorithm for sparse PCA.
In Fig.~\ref{fig:curves}, we sketch the main idea of our elimination step.

{\bf The essentials of the elimination.} First note that a locally optimal support set $\mathcal{I}_k( V c)$ for a fixed $ c$ in \eqref{cVxd}, corresponds to the top $k$ elements of $ v_{ c}$.
As we mentioned before, all elements of $ v_ c$ correspond to hypersurfaces $|[ v(\varphi)]_i|$ that are functions of the $d-1$ spherical variables in $\varphi$.
For a fixed $\varphi\in\Phi^{d-1}$, the candidate support set corresponds exactly to the top $k$ (in absolute value) elements in $v_{ c}= v(\varphi)$, or the top $k$  surfaces $|[ v(\varphi)]_i|$ for that particular $\varphi$.
There is a very simple observation here: a surface $|[ v(\varphi)]_i|$ belongs to the set of top $k$ surfaces if $|[ v(\varphi)]_i|$ is below at most $k-1$ other surfaces on that $\varphi$.
If it is below $k$ surfaces at that point $\varphi$, then $|[ v(\varphi)]_i|$ does not belong in the set of $k$ top surfaces.

A second key observation is the following: the only points $\varphi$ that we need to check are the critical intersection points between $d$ surfaces.
For example, we could construct a set $\mathcal{Y}_k$ of all intersection points of all $d$ sets of curves, such that for any point in this set the number of curves above it is at least $k-1$.
In other words, $\mathcal{Y}_k$  defines a boundary: if a curve is above this boundary then it may become a top $k$ curve; if not it can never appear in a candidate set.
This means that we could test each curve against the points in  $\mathcal{Y}_k$ and discard it if its amplitude is less than the amplitudes of all intersection points in $\mathcal{Y}_k$.
However, the above elimination technique implies that we would need to calculate all intersection points on the $n$ surfaces.
Our goal is to use the above idea by serially checking one by one the intersection points of high amplitude curves.

{\bf Elimination algorithm description.}
We use the above ideas, to build our elimination algorithm.
We first compute the norms of each row $\|[V_d]_{:,i}\|_2$ of $V_d$. 
This norm corresponds to the amplitude of $[v(\varphi)]_i$.
Then, we sort all $n$ rows according to their norms.
We first start with the $k+d$ rows of $V_d$ (i.e., surfaces) of highest norm (i.e., amplitude) and compute their ${k+d \choose d}$ intersection points.
For each intersection point, say $\phi$, we compute the number of $|[v(\varphi)]_i|$ surfaces above it.
If there are less than $k-1$ surfaces above an intersection point, then this means that such point is a potential intersection point where a new curve enters a local top $k$ set.
%defining the set of the $k$-th lowest group of surfaces. 
We keep all these points in a set $\mathcal{P}_{k}$.

We then move to the $(k+d+1)$-st surface of highest amplitude; we test it against the minimum amplitude point in $\mathcal{P}_{k}$.
If the amplitude of the $(k+d+1)$-st surface is less than the minimum amplitude point in $\mathcal{P}_{k}$, then we can safely eliminate this surface (i.e., this row of $V_d$), and all surfaces with maximum amplitude smaller than that (i.e., all rows of $V_d$ with norm smaller than the row of interest).
If its amplitude is larger than the amplitude of this point, then we compute the new set of ${k+d+1 \choose d}$ intersection points obtained by adding this new surface.
We check if some of these can be added in $\mathcal{P}_{k}$, using the test of whether there are at most $k-1$ curves above each point. 
We need also re-check all previous points in $\mathcal{P}_{k}$, since some may no longer be eligible to be in the set; if some are not, then we delete them from the set $\mathcal{P}_{k}$.
We then move on the next row of $V_d$, and continue this process until we reach a row with norm less than the minimum amplitude of the points in $\mathcal{P}_k$.

A pseudo-code for our feature elimination algorithm  can be found as Algorithm 1.
In Fig.~\ref{fig:elimination_example}, we give an example of how our elimination works.

\begin{algorithm}[t]
    \caption{Elimination Algorithm.}
\begin{algorithmic}[1]
   \STATE {\bfseries Input:} $k$, $p$, $ V_d = \left[\sqrt{\lambda_{1}} v_1\ldots \sqrt{\lambda_{1}} v_1\right]$
\STATE Initialize $\mathcal{P}_k\leftarrow\emptyset$
% \STATE Compute $\| e^T_i V_d\|$ for all $i$.
\STATE Sort the rows of $ V_d$ in descending order according to their norms $\| e^T_i V_d\|$.
\STATE $\tilde{n}\leftarrow k+d+1$.
\STATE $\tilde{ V}\leftarrow   V_{1:\tilde{n},:}$.
   \FOR{ all $\tilde{n}\choose d$ subsets $(i_1,\ldots, i_d)$ from $\{1,\ldots, \tilde{n}\}$} 
\FOR{all sequences $(b_1,\ldots, b_{d-1})\in\mathcal{B}$} 
   \STATE
%  Calculate \\
$ c \leftarrow \footnotesize{ \text{nullspace}\left(
\left[
\begin{array}{c}
 e^T_{i_1}-b_1\cdot e^T_{i_2}\\
\vdots\\
 e^T_{i_1}-b_{d-1}\cdot e^T_{i_d}\\
\end{array}
\right] V_d\right)}$
\IF{there are $k-1$ elements of $|v_c|$ greater than $| e_1 V_d c|$}
\STATE $\mathcal{P}_k \leftarrow \mathcal{P}_k\cup \{ | e_1 V_d c| \}$
\ENDIF

\IF{ $\| V_{\tilde{n}+1,:}\|<\min\{x\in\mathcal{P}_k\}$ }
\STATE STOP ITERATIONS.
\ENDIF

\STATE $\tilde{n}\leftarrow\tilde{n}+1$
\STATE $\tilde{ V}\leftarrow   V_{1:\tilde{n},:}$

\FOR{ each element $x$ in $\mathcal{P}_k$}
\STATE check the elements $|v_c|$ greater than $x$
\IF{ there are more than $k-1$ }
\STATE discard it
\ENDIF
\ENDFOR

\ENDFOR
 \ENDFOR
 \STATE {\bf Output:} $\tilde{A}_d = \tilde V_d\tilde V^T_d$, where $\tilde V_d$ comprises of the first $\tilde n$ rows of $V_d$ of highest norm.
\end{algorithmic}
\end{algorithm}

\section{Approximation Guarantees}
\label{sec:approx}
In this section, we prove the approximation guarantees for our algorithm.
Let us define two quantities, namely
\begin{equation}
\text{OPT} = \max_{x\in\mathbb{S}_k} x^TAx \;\text{ and }\;\text{OPT}_d = \max_{x\in\mathbb{S}_k} x^TA_dx, \nonumber
\end{equation}
which correspond to the optimal values of the initial maximization under the full-rank matrix $A$ and its rank-$d$ approximation $A_d$, respectively.
Then, we establish the following lemma.

\begin{lem}
Our approximation factor is lower bounded as 
\begin{equation}
\rho_d=\frac{\underset{\mathcal{I} \in \mathcal{S}_d}{\max}\lambda(A_\mathcal{I})}{\lambda_1^{(k)}}\ge \frac{\text{OPT}_d}{\text{OPT}}.
\end{equation}
\end{lem}
\begin{proof}
The first technical fact that we use is that an optimizer vector $x_d$ for $A_d$ (i.e., the one with the maximum performance for $A_d$), can achieve at least the same performance for $A$, i.e., 
$x_d^TAx_d\ge x_d^T A_d x_d$.
The proof is straightforward: since $A$ is PSD, each quadratic form inside the sum $\sum_{i=1}^n\lambda_i x^T v_i v_i^Tx$ is  a positive number. 
Hence, $\sum_{i=1}^n\lambda_i x^T v_i v_i^Tx\ge \sum_{i=1}^d\lambda_i x^T v_i v_i^Tx$, for any vector $x$ and any $d$.

The second technical fact is that if we are given a vector $x_d$ with nonzero support $\mathcal{I}$, then
calculating $ q_d$, the principal eigenvector of $A_{\mathcal{I}}$, results in a solution for $A$ with better performance compared to $x_d$.
To show that, we first rewrite $x_d$ as $x_d =  P_{\mathcal{I}}x_d$, where $ P_{\mathcal{I}}$ is an $n\times n$ matrix that has 1s on the diagonal elements that multiply the nonzero support of $x_d$ and has 0s elsewhere.
Then, we have 
\begin{align}
\text{OPT}_d\le x_d^TAx_d&=x_d^T P_{\mathcal{I}}A P_{\mathcal{I}}x_d=x_d^T A_{\mathcal{I}}x_d\\
\le& \max_{\|x\|_2=1}x^TA_{\mathcal{I}}x= q^T_dA_{\mathcal{I}} q_d= q^T_d A q_d.\nonumber
\end{align}
Using the above fact for all sets $\mathcal{I}\in\mathcal{S}_d$, we obtain that $\underset{\mathcal{I} \in \mathcal{S}_d}{\max}\lambda(A_\mathcal{I})\ge \text{OPT}_d$, which proves our lower bound.
\end{proof}

{\bf Sparse spectral ratio.} 
A basic quantity that is important in our approximation ratio as we see in the following, is what we define as the {\it sparse spectral ratio}, which is equal to  $\lambda_{d+1}/\lambda_1^{(k)}$.
This ratio will be shown to be directly related to the (non-sparse) spectrum of $A$.

Here we prove the the following lemma.

\begin{lem}
Our approximation ratio is lower bounded as follows. 
\begin{align}
\rho_d \ge 1-\frac{\lambda_{d+1}}{\lambda_1^{(k)}}.
\end{align}
\label{lem:bound1}
\end{lem}
\begin{proof}

We first decompose the quadratic form in \eqref{xAx} in two parts
{\small
\begin{align}
x^TAx=&x^T\left(\sum_{i=1}^n\lambda_i v_i v^T_i\right)x
=x^TA_dx+x^TA_{d^c}x
% =&\left\| V^T_{d}x\right\|^2+\left\| V^T_{d^{\text{c}}}x\right\|^2,
\label{xAx2parts}
\end{align}
}where $A_{d^c} = A-A_d=\sum_{i=d+1}^n \lambda_i v_i v_i^T $.
%  is a matrix that consists of the first $d$ columns of $ V$ and $ V_{d^{\text{c}}}$ consists of the remaining $n-d$ columns of $ V$.
Then, we take maximizations on both parts of \eqref{xAx2parts} over our feasible set of vectors with unity $\ell_2$-norm and cardinality $k$ and obtain
\begin{align}
&\max_{x\in \mathbb{S}_k}x^TAx= \max_{x\in \mathbb{S}_k}\left(x^TA_{d}x+x^TA_{d^c}x\right)\nonumber\\ 
\overset{(i)}{\Leftrightarrow}&\max_{x\in \mathbb{S}_k}x^TAx\le \max_{x\in \mathbb{S}_k}x^TA_{d^c}x+\max_{x\in \mathbb{S}_k}x^TA_{d^c}x\nonumber\\
\Leftrightarrow&\text{OPT}\le \text{OPT}_d+\max_{x\in\mathbb{S}_k}x^TA_{d^c}x\nonumber\\
\overset{(ii)}{\Leftrightarrow}&\text{OPT}\le \text{OPT}_d+\max_{\|x\|_2=1}x^TA_{d^c}x\nonumber\\
\overset{(iii)}{\Leftrightarrow}&\text{OPT}\le \text{OPT}_d+\lambda_{d+1},
\label{ineq1}
\end{align}
where
$(i)$ comes from the fact that the maximum of the sum of two quantities is always upper bounded by the sum of their maximum possible values, $(ii)$ is due to the fact that we lift the $\ell_0$ constraint on the optimizing vector $x$, and $(iii)$ is due to the fact that the largest eigenvalue of $A-A_d$ is equal to $\lambda_{d+1}$.
Moreover, due to the fact that $\text{OPT}\ge \text{OPT}_d$,
% : let $x_d$ be an optimal solution to $\max_{x\in\mathbb{S}_k}x^TA_dx$. 
we have
\begin{equation}
\text{OPT}-\lambda_{d+1}\le \text{OPT}_d\le \text{OPT}.
\label{approxadd}
\end{equation}
Dividing both the terms of \eqref{approxadd} with $\text{OPT}$ yields
\begin{align}
1-\frac{\lambda_{d+1}}{\lambda_1^{(k)}}=1-\frac{\lambda_{d+1}}{\text{OPT}}\le \frac{\text{OPT}_d}{\text{OPT}} \le \rho_d\le 1.
\end{align}
\end{proof}

{\bf Lower-bounding $\lambda_1^{(k)}$.}
We will now give two lower-bounds on $\text{OPT}$.
\begin{lem}
The sparse eigenvalue of $A$ is lower bounded as 
\begin{equation}
\lambda_1^{(k)} \ge \max\left\{\frac{k}{n} \lambda_1, \lambda_1^{(1)}\right\}.
\end{equation}

\end{lem}

\begin{proof}

The second bound is straightforward: if we assume the feasible solution $ e_{\max}$, being the column of the identity matrix which has a $1$ in the same position as the maximum diagonal element of $A$,
then we get
\begin{equation}
 \text{OPT}\ge  e^T_{\max}C e_{\max}^T=\max_{i=1,\ldots, n}A_{ii} = \lambda_1^{(1)}.
% =\max_{i\in[n]}\deg(\text{word}_i).
\end{equation}
The first bound for $\text{OPT}$ will be obtained by examining the rank-$1$ optimal solution on $A_1$.
Observe that 
\begin{align}
\text{OPT}&\ge \text{OPT}_1 = \max_{x\in \mathbb{S}_k} x^TA_1x\nonumber \\
& = \lambda_1 \max_{x\in \mathbb{S}_k} ( v_1^Tx)^2.
\end{align}
Both $ v_1$ and $x$ have unit norm; this means that $( v_1^Tx)^2\le 1$.
The optimal solution for this problem is to allocate all $k$ nonzero elements of $x$ on $\mathcal{I}_k( v_1)$: 
the top $k$ absolute elements of $ v_1$. 
An optimal solution vector, will give a metric of $( v_1^Tx)^2 = \|[ v_{1}]_{\mathcal{I}_k( v_1)}\|_2^2$.
The norm of the $k$ largest elements of $ v_1$ is then at least equal to $\frac{k}{n}$ times the norm of $ v_1$.
Therefore, we have
\begin{equation}
\text{OPT} \ge \max\left\{\frac{k}{n} \lambda_1, \lambda_1^{(1)}\right\}.
\end{equation}
\end{proof}
The above lemmata can be combined to establish Theorem 1.

\section{Resolving singularities}
In our algorithmic developments, we have made an assumption on the curves studied, i.e., on the rows of the $V_d$ matrix.
This assumption was made so that
tie-breaking cases are evaded, where more than $d$ curves intersect in a single point in the $d$ dimensional space $\Phi^d$.
Such a singularity is possible even for full-rank matrices $V_d$ and can produce enumerating issues in the generation of locally optimal candidate vectors that are obtained through the intersection equations:
\begin{equation}
\left[
\begin{array}{c}
 e^T_{i_1}-b_1 e^T_{i_2}\\
\vdots\\
 e^T_{i_1}-b_{d-1} e^T_{i_d}
\end{array}
\right]_{{d-1}\times n} V_d c = 0_{d-1\times 1}.
\end{equation}
The above requirement can be formalized as: no system of equations of the following form has a nontrivial (i.e., nonzero) solution
\begin{equation}
\left[
\begin{array}{c}
 e^T_{i_1}-b_1 e^T_{i_2}\\
\vdots\\
 e^T_{i_1}-b_{d-1} e^T_{i_d}\\
e^T_{i_1}-b_{d-1} e^T_{i_{d+1}}
\end{array}
\right]_{d\times n} V_d c \ne 0_{d\times 1}
\end{equation}
for all $c\ne 0$ and all possible $d+1$ row indices $i_1,\ldots, i_{d+1}$ (where two indices cannot be the same).
We show here that the above issues can be avoided by slightly perturbing the matrix $V_d$.
We will also show that this perturbation is not changing the approximation guarantees of our scheme, guaranteed that it is sufficiently small.
We can thus rewrite our requirement as a full-rank assumption on the following matrices
\begin{equation}
\rank\left(\left[
\begin{array}{c}
 e^T_{i_1}-b_1 e^T_{i_2}\\
\vdots\\
 e^T_{i_1}-b_{d-1} e^T_{i_d}\\
e^T_{i_1}-b_{d-1} e^T_{i_{d+1}}
\end{array}
\right]_{d\times n} V_d\right)=d
\end{equation}
for all $i_1\ne i_2\ne \ldots \ne i_d$.
Observe that we can rewrite the above matrix as 
\begin{equation}
\left[
\begin{array}{c}
 e^T_{i_1}-b_1 e^T_{i_2}\\
\vdots\\
 e^T_{i_1}-b_{d-1} e^T_{i_d}\\
e^T_{i_1}-b_{d-1} e^T_{i_{d+1}}
\end{array}
\right]_{d\times n}\hspace{-0.5cm} V_d=\left[
\begin{array}{c}
\hspace{0.1cm}[V_d]_{i_1,:}-b_1 [V_d]_{i_2,:}\\
\vdots\\
\hspace{0.1cm}[V_d]_{i_1,:}-b_1 [V_d]_{i_d,:}\\
\hspace{0.1cm}[V_d]_{i_1,:}-b_1 [V_d]_{i_{d+1},:}
\end{array}
\right]
=\left[
\begin{array}{c}
\hspace{0.1cm}[V_d]_{i_1,:}\\
\vdots\\
\hspace{0.1cm}[V_d]_{i_1,:}\\
\hspace{0.1cm}[V_d]_{i_1,:}
\end{array}
\right]-\left[
\begin{array}{c}
b_1 [V_d]_{i_2,:}\\
\vdots\\
b_1 [V_d]_{i_d,:}\\
b_1 [V_d]_{i_{d+1},:}
\end{array}
\right] = {R_{i_1}}+G_{i_2,\ldots,i_{d+1}}\nonumber
\end{equation}
where $R_{i_1}$ is a rank-1 matrix.
Observe that the rank of the above matrix depends on the ranks of both of its components and how the two subspaces interact.
It should not be hard to see that we can add a $d\times d$ random matrix $\Delta = [\delta_1 \delta_2 \ldots \delta_d]$ to the above matrix, so that ${R_{i_1}}+G_{i_2,\ldots,i_{d+1}}+\Delta$ is full-rank with probability 1.

Let $E_d$ be an $n\times d$ matrix with entries that are uniformly distributed and bounded as $|E_{i,j}|\le \epsilon$.
Instead of working on $V_d$ we will work on the perturbed matrix $\tilde V_d = V_d+E_d$.
Then, observe that for any matrix of the previous form ${R_{i_1}}+G_{i_2,\ldots,i_{d+1}}$ we now have
${R_{i_1}}+G_{i_2,\ldots,i_{d+1}}+[E_d]_{i_1,:}\otimes 1_{d\times 1}+E_{i_2,\ldots, i_{d+1}}$, where
\begin{equation}
E_{i_2,\ldots, i_d} = \left[
\begin{array}{c}
\hspace{0.01cm}[E_d]_{i_2,:}\\
\hspace{0.01cm}[E_d]_{i_3,:}\\
\vdots\\
\hspace{0.01cm}[E_d]_{i_{d+1},:}
\end{array}
\right].
\end{equation}
Conditioned on the randomness of $[E_d]_{i_1,:}$, the matrix ${R_{i_1}}+G_{i_2,\ldots,i_{d+1}}+[E_d]_{i_1,:}\otimes 1_{d\times 1}+E_{i_2,\ldots, i_{d+1}}$ is full rank. Then due to the fact that there are $d$ random variables in $[E_d]_{i_1,:}$ and $d^2$ random variable in ${R_{i_1}}+G_{i_2,\ldots,i_{d+1}}+[E_d]_{i_1,:}\otimes 1_{d\times 1}+E_{i_2,\ldots, i_{d+1}}$, the latter matrix will be full-rank with probability 1 using a union bounding argument.
This means that all ${n \choose d}$ submatrices of $\tilde V_d$ will be full rank, hence obtaining the property that no $d+1$ curves intersect in a single point in $\Phi^d$.

Now we will show that this small perturbation does not change our metric of interest significantly.
The following holds for any unit norm vector $x$
\begin{align}
x^T(V_d+E_d)(V_d+E_d)^Tx &= x^TV_dV_d^Tx+x^TE_dE_d^Tx+2x^TV_dE_d^Tx\nonumber
 \ge  x^TV_dV_d^Tx+2x^TV_dE_d^Tx \nonumber\\
& \ge  x^TV_dV_d^Tx-2\|V_d^Tx\|\cdot  \|E_d^Tx\| \nonumber 
 \ge  x^TV_dV_d^Tx-2\sqrt{\lambda_1\cdot \lambda_1(E_dE_d^T)} \nonumber
%\ge 
%\left|\|V_d^Tx\|-\|E_dx\|\right|\ge 
%%  &\le \nonumber\\
%\left|\|V_d^Tx\|-\sqrt{\lambda_1(E_dE_d^T)}\right| %&\le\|(V_d+E_d)^Tx\|^2\nonumber\\
%\ge \left|\|V_d^Tx\|^2-\frac{1}{\log n}\right|
%  &\le \|(V_d+E_d)^Tx\|^2\nonumber
\end{align}
and
{\small
\begin{align}
\|(V_d+E_d)^Tx\|^2  \le \|V_d^Tx\|^2+2\|E_dx\|\|V_d^Tx\|+\|E_d^Tx\|^2\nonumber &\le \|V_d^Tx\|^2+2\sqrt{\lambda_1\cdot \lambda_1(E_dE_d^T)}+\lambda_1(E_dE_d^T)\nonumber\\
& \le \|V_d^Tx\|^2+3\sqrt{\lambda_1\cdot \lambda_1(E_dE_d^T)}.\nonumber
\end{align}
}Combining the above we obtain the following bound
\begin{align}
&  x^TV_dV_d^Tx-3\sqrt{\lambda_1\cdot \lambda_1(E_dE_d^T)} \nonumber
\le \|(V_d+E_d)^Tx\|^2 
 \le \|V_d^Tx\|^2+3\sqrt{\lambda_1\cdot \lambda_1(E_dE_d^T)}.\nonumber
\end{align}
By the above, we can appropriately pick $\epsilon$ such that $3\sqrt{\lambda_1\cdot \lambda_1(E_dE_d^T)}=o(1)$.
An easy way to get a bound on is via the Gershgorin circle theorem \citep{horn1990matrix}, which yields $\lambda_1(E_dE_d^T)<nd\cdot \epsilon^2$.
Hence, an $\epsilon<\frac{1}{\sqrt{\lambda_1nd}\log n}$ works for our purpose.

To summarize, in the above we show that there is an easy way to avoid singularities in our problem.
Instead of solving the original rank-$d$ problem on $V_d$, we can instead solve it on $V_d+E_d$, with an $E_d$ random matrix with sufficiently small entries.
This slight perturbation  only incurs an error of at most $\frac{1}{\log n}$ in the objective, which asymptotically becomes zero as $n$ increases.

\section{Twitter data-set description}

In Table~\ref{table:data-set}, we give an overview of our Twitter data set.
\begin{table}[h]
\begin{center}
{\footnotesize
\begin{tabular}{c|l}
\hline
\multicolumn{2}{c}{Data-set Specifications}\\
\hline
\hline
% Tweets crawled & based on IP or use of Greek chars\\
% \hline
Geography & mostly Greece\\
\hline
time-window & January 1-August 20, 2011\\
\hline 
unique user IDs & $\sim 120$k\\
\hline
size & $\sim$10 million entries\\
\hline
tweets/month & $\sim 1.5$ million\\
\hline
tweets/day & $\sim70$k\\
\hline
tweets/hour & $\sim3$k\\
\hline
words/tweet & $\sim5$\\
\hline
character limit/tweet & $140$\\
\hline
\end{tabular}
}
\end{center}
\caption{Data-set specifications.}
\label{table:data-set}
\end{table}

\subsection{Power Laws}
\begin{figure*}[h]
  \centering
   \subfigure[The plots that we provide are for hour-length, day-length, and month-length analysis, and subsets of our data set based on a specific query.]{\includegraphics[scale=0.47]{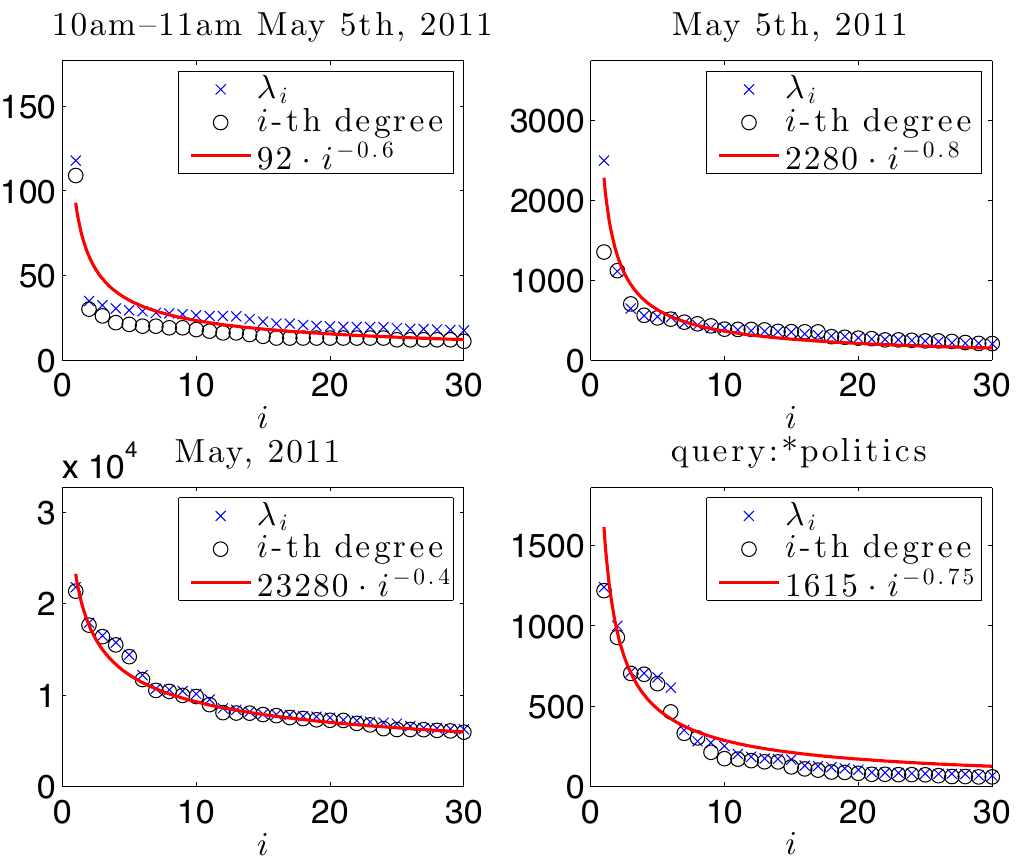}}\quad
   \subfigure[Approximation Guarantees:  we show how the approximation guarantees for these specific subsets of our data set scales with $d$.]{\includegraphics[scale=0.47]{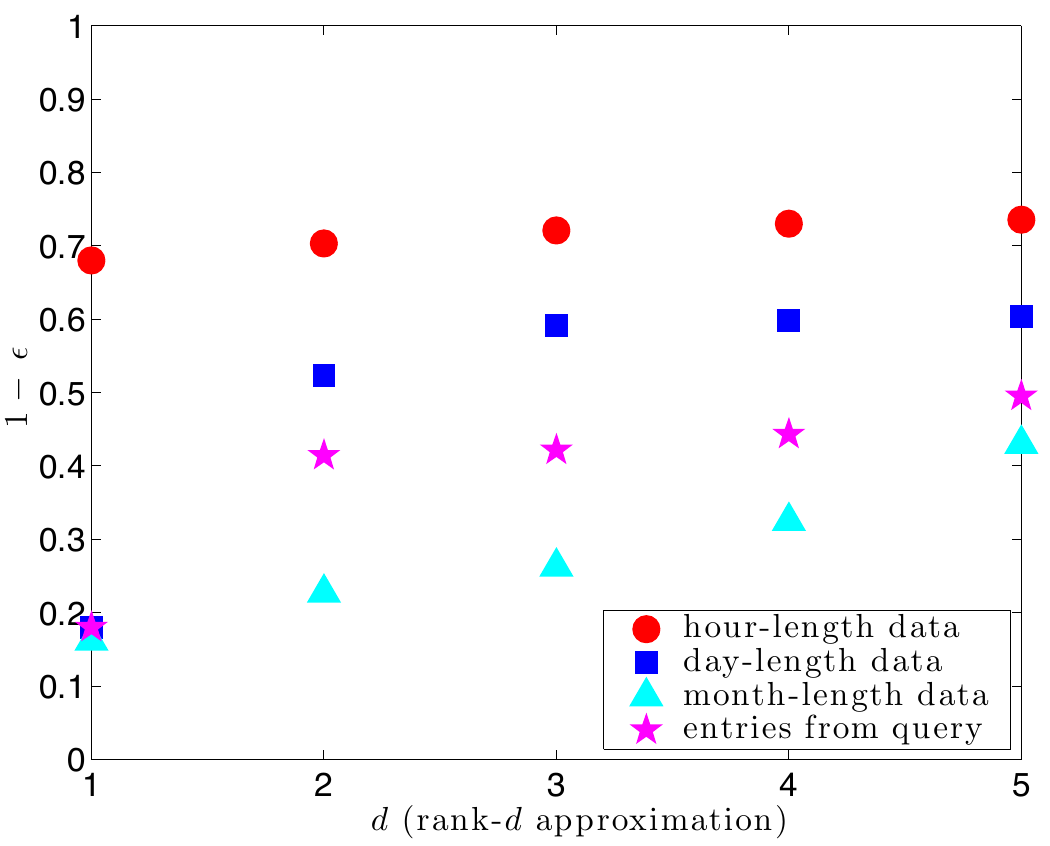}}
\label{fig:power_laws}
\caption{Power laws and approximation guarantees}
\label{fig:powerlaw}
\end{figure*}

In this subsection we provide empirical evidence that our tested data-sets exhibit a power law decay on their spectrum
We report these observations as a proof of concept for our approximation guarantees.
Based on the spectrum of some subsets of our data-set, we provide the exact approximation guarantees derived using our bounds.

In Fig.~\ref{fig:powerlaw}, we plot the best fit power law for the spectrum and degrees with data-set parameters given on the figures.
The plots that we provide are for hour-length, day-length, and month-length analysis, and subsets of our data set based on a specific query.
We observe that for all these subsets of our data set, the spectrum indeed follows a power-law.
An interesting observation is that a very similar power law is followed by the degrees of the terms in the data set.
This finding is compatible to the generative models and analysis of \citep{mihail2002eigenvalue, chung2003eigenvalues}.
The rough overview is that eigenvalues of $A$ can be well approximated using the diagonal elements of $A$.
In the same figure, we show how our approximation guarantees that based on the spectrum of $A$ scales with $d$, for the various data-sets tested.
We only plot for $d$ up to $5$, since for any larger $d$ our algorithm becomes impractical for moderately large small data sets.

\end{appendix}

\end{document}